\documentclass{article}

\usepackage[main, final]{neurips_2026}

\usepackage[utf8]{inputenc} 
\usepackage[T1]{fontenc} 
\usepackage{hyperref} 
\usepackage{url} 
\usepackage{booktabs} 
\usepackage{amsfonts} 
\usepackage{nicefrac} 
\usepackage{microtype} 
\usepackage{xcolor} 

\usepackage{algorithm}
\usepackage{algorithmic}
\usepackage{amsmath}
\usepackage{mathtools}
\usepackage{amsthm}
\usepackage{graphicx}

\theoremstyle{plain}
\newtheorem{theorem}{Theorem}

\newtheorem{lemma}{Lemma}

\theoremstyle{definition}

\newtheorem{assumption}{Assumption}
\theoremstyle{remark}

\title{Actor-Accelerated Policy Dual Averaging for Reinforcement Learning in Continuous Action Spaces}

\author{%
  Ji~Gao\textsuperscript{1} \\
  \texttt{jgao325@gatech.edu} \\
  \And
  Caleb~Ju\textsuperscript{2} \\
  \texttt{calebju4@gatech.edu} \\
  \AND
  Guanghui~Lan\textsuperscript{2} \\
  \texttt{george.lan@isye.gatech.edu} \\
  \And
  Zhaohui~Tong\textsuperscript{1} \\
  \texttt{zt7@gatech.edu} \\
  \AND
  \normalfont
  \textsuperscript{1}School of Chemical and Biomolecular Engineering \\
  \textsuperscript{2}H. Milton Stewart School of Industrial and Systems Engineering \\
  Georgia Institute of Technology, Atlanta, GA, USA \\
}

\begin{document}

\maketitle

\begin{abstract}
    Policy Dual Averaging (PDA) offers a principled Policy Mirror Descent (PMD) framework that more naturally admits value function approximation than standard PMD, enabling the use of approximate advantage (or Q-) functions while retaining strong convergence guarantees. However, applying PDA in continuous state and action spaces remains computationally challenging, since action selection involves solving an optimization subproblem at each decision step. In this paper, we propose \textit{actor-accelerated PDA} (AA-PDA), which uses a learned policy network to approximate the solution of the optimization subproblems, significantly reducing runtime. We provide a theoretical analysis that quantifies how actor approximation error impacts the convergence of PDA under suitable assumptions. We then evaluate its performance on several benchmarks in robotics, control, and operations research problems. AA-PDA achieves better performance compared to on-policy baselines such as Proximal Policy Optimization (PPO). Overall, our results take a step toward bridging the gap between the theoretical advantages of PDA and its practical deployment in continuous-action problems with function approximation. Our code is available at \url{https://github.com/jgioa/pdarl}.
\end{abstract}

\section{Introduction}
Many advances in reinforcement learning (RL) use parameterized policy-gradient methods, such as Trust Region Policy Optimization (TRPO)~\citep{schulmanTrustRegionPolicy2015} and Proximal Policy Optimization (PPO)~\citep{schulmanProximalPolicyOptimization2017}. These methods have achieved great success in continuous control tasks in various disciplines~\citep{duanBenchmarkingDeepReinforcement2016, lillicrapContinuousControlDeep2019}. Their practical performance has been partially supported by recent theoretical analysis. In particular, Policy Mirror Descent (PMD) serves as a unifying framework for first-order RL methods with convergence guarantees~\citep{lanPolicyMirrorDescent2023, zhanPolicyMirrorDescent2023, tomarMirrorDescentPolicy2021, alfanoNovelFrameworkPolicy2023, lanBlockPolicyMirror2023, liPolicyMirrorDescent2025, liHomotopicPolicyMirror2024,neuUnifiedViewEntropyregularized2017}.
  
Despite this progress, theoretically principled policy updates may require solving optimization problems upon request to obtain actions in continuous action spaces, and neural-network-based parameterization can further complicate the problem. Within PMD-style methods for continuous actions, the theoretically defined update may require solving an optimization problem whenever a policy action is queried. Nonlinear value and policy approximations can make this action-generation subproblem non-convex or poorly conditioned~\citep{fan2020theoretical,liu2019neural,juPolicyOptimizationGeneral2024,alfanoNovelFrameworkPolicy2023,agarwal2021theory}. Large penalty coefficients used to limit policy changes can further increase the objective's Lipschitz constant, making repeated action solves slow or unreliable.

Recently, Policy Dual Averaging (PDA), inspired by Nesterov's Dual Averaging algorithm~\citep{nesterovPrimaldualSubgradientMethods2009}, was proposed as a promising alternative that matches PMD's convergence guarantees and offers a new perspective on policy optimization~\citep{juPolicyOptimizationGeneral2024}. Through careful algorithm design, PDA's optimization subproblems are made weakly convex while avoiding policy function approximation. As a result, locally (possibly globally) optimal solutions for the subproblems can be efficiently solved while evading explicit policy parameterization. While this leads to efficiently solvable policy update steps, the policy evaluation step becomes the main bottleneck. Each action query during policy evaluation requires solving a separate optimization subproblem, and this per-decision cost is often prohibitively expensive when implemented directly~\citep{juPolicyOptimizationGeneral2024, lanPolicyMirrorDescent2023}. See Section~\ref{sec:pda_intro} for more details about PDA.

Building upon PDA, we introduce \textit{actor-accelerated PDA} (AA-PDA). AA-PDA changes only action generation by using a learned policy network to approximate the solution of the expensive optimization subproblem, accelerating the action computation during policy evaluation. The main contributions of this paper consist of the following aspects: \textbf{1) Practical Framework:} AA-PDA has a simple practical implementation, with only two algorithm-specific hyperparameters for regularization and exploration in addition to the standard hyperparameters for deep RL.  It uses one actor forward pass at decision time and standard first-order minibatch training, and preserves a direct link to the PDA objective and convergence analysis; \textbf{2) Convergence and Error Analysis:} We conduct theoretical analysis that characterizes how approximation errors from actor-based subproblem solutions impact the overall convergence and optimality; \textbf{3) Empirical Performance:} Across continuous-control benchmarks, AA-PDA outperforms on-policy methods such as PPO in aggregate normalized performance at comparable training cost, and approaches the performance of the established off-policy method TD3.

\section{Related Work}

\textbf{Policy Mirror Descent.} PMD provides a unifying framework for policy optimization, connecting policy-gradient, trust-region, and proximal methods under a common theoretical perspective. Existing work establishes convergence guarantees under suitable regularity conditions, typically assuming exact or sufficiently accurate policy evaluation~\citep{lanPolicyMirrorDescent2023, zhanPolicyMirrorDescent2023, alfanoNovelFrameworkPolicy2023}. PMD and PDA are closely related first-order frameworks, but differ in how they aggregate and regularize updates: PMD uses the current advantage estimate with a regularizer centered at the current policy, whereas PDA accumulates advantage information across iterations and regularizes around a fixed initial policy or prox-center.

This distinction also affects implementation in continuous spaces. In PMD, the prox-center changes at every iteration, so with parameterized policies and value functions the algorithm must repeatedly represent and optimize a new objective tied to the current policy. In PDA, the reference policy remains fixed and the cumulative objective can be maintained explicitly. Both updates still require solving statewise optimization subproblems in continuous action spaces, which motivates the action approximation developed in AA-PDA.

\textbf{Regularized and Trust-Region Policy Optimization.} Regularization is central to many practical policy optimization algorithms. TRPO~\citep{schulmanTrustRegionPolicy2015} formulates policy updates as a constrained or penalized optimization problem to guarantee monotonic improvement. PPO~\citep{schulmanProximalPolicyOptimization2017} applies clipping to a surrogate objective as an approximation. These methods rely on a linear approximation of the cost function in the PMD framework, where the approximation error is bounded by the Kullback--Leibler (KL) divergence between the old and new policies. To ensure monotonic improvement, the update step can be formulated as minimizing the objective subject to a penalty on the maximum statewise KL divergence. TRPO relaxes the impractical maximum KL constraint to an average KL constraint. PPO introduces a ratio clipping mechanism to bound the likelihood ratio $\pi(a|s)/\pi_{k}(a|s)$, which helps mitigate the impact of heavy-tailed noise in stochastic advantage estimates. Furthermore, the PMD framework generalizes these formulations by allowing for arbitrary Bregman divergences and strongly convex regularizers.

\section{Preliminaries}

We study RL problems with continuous state and action spaces, modeled as infinite-horizon discounted Markov decision processes (MDPs) defined by a five-tuple \((\mathbb{S}, \mathcal{A}, \mathbb{P}, c, \gamma)\). \(\mathbb{S} \subseteq \mathbb{R}^{n_s}\) denotes the state space and \(\mathcal{A} \subseteq \mathbb{R}^{n_u}\) denotes the closed and convex action space. The transition dynamics are specified by a stochastic transition model \(\mathbb{P} : \mathfrak{B}(\mathbb{S}) \times \mathbb{S} \times \mathcal{A} \to [0,1]\), which assigns a probability measure over next states for each state-action pair. A deterministic policy \(\pi : \mathbb{S} \to \mathcal{A}\) selects a feasible action for a given state. \(\gamma \in [0,1)\) is the discount factor. \(c : \mathbb{S} \times \mathcal{A} \to \mathbb{R}\) is the cost function. We keep the cost-minimization convention of the original PDA analysis~\citep{juPolicyOptimizationGeneral2024}. Reward maximization is the equivalent problem with \(c=-r\) for a reward \(r\), and our experiments use the environments' rewards, which corresponds to this situation.

For a given policy \(\pi\), its performance is quantified by either the action-value function ($Q^{\pi}(s,a)$) or the state-value function ($V^{\pi}(s)$):
\begin{equation} \label{eq:Q_defn}
Q^{\pi}(s,a)
:= \mathbb{E}_{\pi} \left[
\sum_{t=0}^{\infty} \gamma^t c(s_t,a_t) 
\,\middle|\, s_0=s, a_0=a
\right],\quad 
V^{\pi}(s)
:= \mathbb{E}_{\pi} \left[
\sum_{t=0}^{\infty} \gamma^t c(s_t,a_t) 
\,\middle|\, s_0=s
\right].
\end{equation}

Note that our results can be extended to the regularized case, i.e.,~the cost $c(s_t,a_t)$ is replaced by a regularized cost $c(s_t,a_t) + h^{a_t}(s_t)$ for a convex regularizer $h^\cdot(s) : \mathcal{A} \to \mathbb{R}$ in~\eqref{eq:Q_defn}. We only study the un-regularized case since it simplifies the analysis, and our numerical results do not use such regularizers.
These definitions imply the identities $V^\pi(s) = Q^\pi(s,\pi(s))$ and
\begin{equation}
Q^{\pi}(s,a)
= c(s,a)
+ \gamma \int_{\mathbb{S}} V^{\pi}(s') \, \mathbb{P}(ds' \mid s,a).
\end{equation}
The objective of the RL problem is to find a policy that minimizes the expected discounted cost under an initial state distribution \(\rho\):
\begin{equation} \label{eq:np}
\min_{\pi} \; f_\rho(\pi),
\quad
f_\rho(\pi) := \int_{\mathbb{S}} V^{\pi}(s)\, \rho(ds).
\end{equation}
In this paper, we let $\rho$ be an arbitrary distribution over states $\mathbb{S}$ (see Appendix \ref{sec:rho} for details). 
Here, an optimal policy \(\pi^*\) satisfies \(V^{\pi^*}(s) \le V^{\pi}(s)\) for all \(s \in \mathbb{S}\) and all feasible policies \(\pi\). To analyze policy improvement, we define the discounted state visitation measure following a policy \(\pi\). For an initial state \(s\) and any measurable set \(B \subseteq \mathbb{S}\), this measure is given by
\begin{equation} \label{eq:disc_fvisit}
\kappa_s^{\pi}(B)
:= (1-\gamma) \sum_{t=0}^{\infty} \gamma^t
\Pr^{\pi}(s_t \in B \mid s_0=s),
\end{equation}
where \(\Pr^{\pi}(s_t \in \cdot \mid s_0=s)\) denotes the state distribution of \(s_t\) following the policy \(\pi\) with initial state \(s_0=s\). 

The performance difference lemma relates the performance of two policies~\citep{kakade2002approximately,juPolicyOptimizationGeneral2024}. For any pair of feasible policies \(\pi\) and \(\pi'\), the difference in their value functions can be written as follows, where $\psi^{\pi}$ is the advantage function.
\begin{equation}
\label{eq:performance_difference}
V^{\pi'}(s) - V^{\pi}(s)
= \frac{1}{1-\gamma}
\int_{\mathbb{S}} \psi^{\pi}(q, \pi'(q))\, \kappa_s^{\pi'}(dq),
\quad
\text{where}\
\psi^{\pi}(q, a)
:= Q^{\pi}(q,a) - V^{\pi}(q).
\end{equation}

\subsection{Policy Dual Averaging}
\label{sec:pda_intro}

\begin{algorithm}
\caption{Policy dual averaging}
\begin{algorithmic}[1]
\label{alg:pda_original}
\STATE {\bfseries Input:} Weights $\beta_k \ge 0$, step sizes $\lambda_k \ge 0$, initial policy $\pi_0$.
\FOR{$k=0, 1, \dots$}
  \STATE $\pi_{k+1}(s)=\arg\min_{a\in\mathcal{A}} \sum_{t=0}^k\beta_t\psi^{\pi_t}(s,a) + \lambda_k D(\pi_0(s),a)$
\ENDFOR
\end{algorithmic}
\end{algorithm}

We first briefly introduce the original PDA algorithm (Algorithm \ref{alg:pda_original}). Each PDA iteration defines an updated policy $\pi_{k+1}(s)$ as the minimizer of an optimization subproblem involving a weighted sum of past advantage functions plus a regularization term $\lambda_k D(\pi_0(s),a)$. The theory for dual averaging will set the scalar weights $\beta_t$ to be an increasing sequence so that more recent advantage functions are more influential. Here, $D(\cdot, \cdot)$ is a Bregman divergence (or prox-function) generated by a (without loss of generality by scaling) 1-strongly convex distance generating function $\omega : \mathcal A \to \mathbb{R}$ with respect to a norm $\|\cdot\|$:
\begin{equation}
D(a_2, a_1):=\omega(a_1)-[\omega(a_2) + \langle \omega'(a_2), a_1 - a_2\rangle] \geq \|a_1 - a_2\|^2/2,\quad \forall a_1, a_2 \in \mathcal{A}
\end{equation}

The choice of $\pi_0(s)$ (the prox-center or anchor point) instead of recent $\pi_k(s)$ will be beneficial later when incorporating function approximation like neural networks. 

Similar to PMD, PDA has theoretical guarantees.
When the advantage function can be exactly computed, then PDA converges to the globally optimal policy at a linear rate. 
In the stochastic setting and under standard statistical assumptions, PDA can still converge to global optimality with sample complexity of $O(\epsilon^{-2})$.
We also mention that when using strongly convex regularizations $h^\cdot(s)$, such as entropy regularization, PDA has an improved $O(\epsilon^{-1})$ sample complexity.
See~\cite{juPolicyOptimizationGeneral2024} for more details.

\section{Actor-Accelerated Policy Dual Averaging}
\label{sec:method}

PDA and AA-PDA use the same cumulative regularized PDA objective and weighted advantage history. In the original PDA formulation, the policy is represented implicitly rather than by an explicit parametric actor, so there is no need to solve the optimization problem over the entire state space in advance. Instead, the action is obtained by solving the corresponding optimization subproblem only when the policy is queried at a particular state. AA-PDA changes only this action-generation mechanism: it trains an actor to approximate the same statewise minimizer, replacing query-time optimization with a single actor forward pass during rollout. This also enables AA-PDA to use standard policy evaluation procedures, the same ones as in PPO.

Algorithm~\ref{alg:pda} distinguishes the PDA action from its actor approximation. For each queried state $s$, the exact action $\pi_{k+1}(s)=\arg\min_{a\in\mathcal{A}}\tilde\Psi_k(s,a)$ minimizes the cumulative regularized action objective. Throughout the paper, we refer to this statewise problem as the \emph{optimization subproblem}. Rather than solving the subproblem whenever the policy is queried, AA-PDA trains an actor $\hat\pi_{k+1}(s;\theta_k^\pi)$ to approximate its minimizer, where $\theta_k^\pi$ denotes the actor parameters obtained by fitting the iteration-$k$ objective $\tilde\Psi_k$. For the analysis, let $\tilde{\psi}$ denote the advantage function parameterized by $\theta_k$, and define the state-averaged Bregman divergence originating at state $s$ as $\mathcal{D}_s(\pi, \pi'):=\mathbb{E}_{q\sim \kappa_s^{\pi^*}}[D(\pi(q), \pi'(q))]$. The step sizes $\beta_k$ and $\lambda_k$ in Algorithm~\ref{alg:pda} are specified in Theorems~\ref{theorem:convex} and~\ref{theorem:non-convex}.

\begin{algorithm}[H]
  \caption{Actor-accelerated PDA with function approximation}
  \label{alg:pda}
  \begin{algorithmic}[1]
    \STATE {\bfseries Input:} Weights $\beta_k \ge 0$, step sizes $\lambda_k \ge 0$, initial policy $\pi_0$, with $\hat\pi_0:=\pi_0$.
    
    \FOR{$k=0, 1, \dots$}
      \STATE Define an approximate advantage function: 
      \STATE \quad $\tilde{\psi}(s, a; \theta_k) \approx \psi^{\hat{\pi}_k}(s,a) := Q^{\hat\pi_k}(s, a) - V^{\hat\pi_k}(s)$
      \STATE Define the cumulative objective as:
      \STATE \quad $\tilde{\Psi}_k(s, a) := \sum_{t=0}^k \beta_t \tilde{\psi}(s, a; \theta_t) + \lambda_k D(\pi_0(s), a)$
      \STATE Action-computation \textit{optimization subproblem:}
      \STATE $\displaystyle \hat\pi_{k+1}(s;\theta^\pi_k) \approx \pi_{k+1}(s) =\arg\min_{a \in \mathcal{A}} \tilde{\Psi}_k(s, a)$
    \ENDFOR
  \end{algorithmic}
\end{algorithm}

\section{Convergence and Error Analysis}
\label{sec:convergence}
We analyze the convergence of actor-accelerated PDA in the presence of both function approximation error in the advantage function $\tilde\psi$ and optimality gap in actor $\hat\pi$ for subproblem solution. We start by placing assumptions on $\tilde{\psi}$.
Let $\xi_k$ denote the random samples generated at iteration $k$ to fit/learn $\tilde \psi(s,a;\theta_k)$. Define the approximation error
\begin{align*}
    \delta_k(s, \cdot) 
    := 
    \tilde{\psi}(s, \cdot; \theta_k) - \psi^{\hat\pi_k}(s, \cdot)
    =
    [\underbrace{\tilde{\psi}(s, \cdot; \theta_k) - \mathbb{E}_{\xi_k}\tilde{\psi}(s, \cdot; \theta_k)}_{\delta^{sto}_k(s,\cdot)}] + [\underbrace{\mathbb{E}_{\xi_k}\tilde{\psi}(s, \cdot; \theta_k) - \psi^{\hat\pi_k}(s, \cdot)}_{\delta_k^{det}(s,\cdot)}],
\end{align*}
The decomposition is additive. The stochastic error $\delta^{sto}_k$ captures statistical estimation error from the samples used at iteration $k$, while the deterministic error $\delta^{det}_k$ contains bias and function-approximation error.
Here $\mathbb{E}_{\xi_k}$ is the expectation w.r.t.~$\xi_k$ conditioned on the filtration $F_{k-1} := \{\xi_0,\ldots,\xi_{k-1}\}$.

\begin{assumption}[Regularity]\label{assumption:1}
We make the following assumptions regarding the advantage function and the objective:
\begin{enumerate}
\item \textbf{Lipschitz continuity}: The approximate advantage function $\tilde{\psi}(s, \cdot; \theta)$ is $M_{\tilde{Q}}$-Lipschitz, meaning $\|\tilde{\psi}(s, a; \theta) - \tilde{\psi}(s, a'; \theta)\| \le M_{\tilde{Q}}\|a - a'\|$ for all $a, a'$. The true advantage function $\psi^{\pi}(s, \cdot)$ is similarly $M_Q$-Lipschitz. Consequently, the approximation error $\delta_k(s, \cdot)$ is $(M_{\tilde{Q}}+M_Q)$-Lipschitz.
\item \textbf{Weak convexity}: 
The approximate advantage function $\tilde{\psi}(s, \cdot; \theta)$ is $\mu_{\tilde Q}$-weakly convex, where $\mu_{\tilde Q} \in \mathbb{R}$, meaning that the function $\tilde{\psi}(s, a; \theta) + \frac{\mu_{\tilde Q}}{2}\|a\|^2$ is convex with respect to $a$. We define $\tilde \mu_d := -\mu_{\tilde Q}$.
\item \textbf{Approximate evaluation error bound}: For all iterations $k$, the expected approximation errors are bounded by a constant $\varsigma > 0$ such that 
$\max_{s \in \mathbb{S}} \mathbb{E}\Big[ \big| \delta_{k-1}^{det}(s, \hat\pi_{k-1}(s)) \big| + \big| \delta_{k-1}^{det}(s, \pi^*(s)) \big| \Big] \le \varsigma$. 
\end{enumerate}
\end{assumption}

An immediate consequence of the ``weak convexity'' assumption is that the cumulative objective $\tilde{\Psi}_k(s,a)$ from Algorithm~\ref{alg:pda} is strongly convex with respect to $a$, with modulus $\tilde{\mu}_k := \tilde{\mu}_d \sum_{t=0}^k \beta_t + \lambda_k$ as long as $\lambda_k$ ensures $\tilde \mu_k \geq 0$. 

We next state an assumption on the approximate updated policy $\hat{\pi}_{k+1}$ from Algorithm~\ref{alg:pda}. This assumption can be viewed as the policy update analogue of the ``Approximate evaluation error bound'' assumption above. This assumption is needed to guarantee the stronger distribution-free convergence~\citep{ju2024strongly}, where the optimality gap is made small across all states simultaneously. In practice, there may be only a finite set of states of interest, such as the start state in MuJoCo environments. It may then be possible to relax this assumption to hold on a subset of states related to them. We leave the analysis of such relaxed conditions to future work.

\begin{assumption}[Optimality Gap]
\label{assumption:2}
The actor $\hat{\pi}_{k+1}$ approximately solves the optimization subproblem with a function value error bounded by $\epsilon_{opt, k}(s)$, which satisfies:
\begin{equation}
\tilde{\Psi}_k(s, \hat{\pi}_{k+1}(s)) \le \tilde{\Psi}_k(s, \pi_{k+1}(s)) + \epsilon_{opt, k}(s), \quad \forall s \in \mathbb{S}, ~k.
\end{equation}
Moreover, there exists a constant $\epsilon$ such that $0\le \epsilon_{opt, k}(s)\le\epsilon$ for all iterations $k$ and states $s \in \mathbb{S}$. 
\end{assumption}

Using the generalized strong convexity with Bregman divergence~\citep{lanDeterministicConvexOptimization2020} of $\tilde{\Psi}_k$, this implies a bound on the distance between the ideal and approximate policies:
\begin{equation}
D(\pi_{k+1}(s), \hat\pi_{k+1}(s)) \le \frac{\epsilon_{opt, k}(s)}{\tilde{\mu}_k}, \quad \forall s \in \mathbb{S}, ~k.
\end{equation}

\paragraph{Discussion on Assumptions.}
\label{sec:assumptions}
Lipschitz continuity holds for smooth neural network architectures, such as fully connected neural networks with smooth activation functions like Tanh~\citep{allen2019convergence}. Weak convexity also holds for such smooth neural networks because their Hessians are bounded~\citep{beck2017first,davis2019stochastic}.
The ``approximate evaluation error bound'' presumes that the function approximator, such as a neural network, can learn the true advantage function up to an error threshold $\varsigma$. The threshold $\varsigma$ is finite when both the costs $c(s,a)$ and neural network outputs are bounded. Furthermore, since neural networks are universal function approximators, $\varsigma$ can possibly be made arbitrarily small when $\psi^{\pi}(\cdot,\cdot)$ is jointly continuous in both inputs (see \citealp{juPolicyOptimizationGeneral2024} for sufficient conditions) and the width or depth of the neural network is sufficiently large. In addition, Appendix~\ref{sec:app:monitor} has further explanations on $\varsigma$ and $\epsilon$.

\subsection{Convergence when \texorpdfstring{$\tilde{\mu}_d \ge 0$}{mu-d >= 0}} \label{sec:cvg_1}

We start with the case of $\tilde{\mu}_d \ge 0$, which means the cumulative advantage term $\sum_{t=0}^k \beta_t \tilde{\psi}(s,a;\theta_t)$ is convex in $a$. In this case, Theorem~\ref{theorem:convex} controls the global performance gap and global convergence requires the residual approximation terms to vanish.

\begin{theorem}
\label{theorem:convex}
Suppose $\lambda_{k+1} \geq \lambda_k \geq 0$ for all $k \geq 0$  and $\lambda_0>0$, with the convention $\lambda_{-1}:=\lambda_0$ and $\epsilon_{opt,-1}:=0$. Then under Assumptions~\ref{assumption:1} and~\ref{assumption:2}, if $\tilde{\mu}_d \ge 0$, the performance gap of the sequence of policies generated by actor-accelerated PDA satisfies:
\begin{equation}
\begin{aligned}
(1-\gamma) \frac{1}{\bar{\beta}_k} \sum_{t=0}^{k-1} \beta_t \mathbb{E}[V^{\hat{\pi}_t}(s) - V^{\pi^*}(s)]
+ \frac{\tilde{\mu}_{k-1}}{\bar{\beta}_k} \mathbb{E}[\mathcal{D}_s(\pi_k, \pi^*)]
\\
\le
\frac{\lambda_{k-1}}{\bar{\beta}_k} \mathcal{D}_s(\pi_0, \pi^*)
+ \frac{M_{\tilde{Q}}^2}{2\bar{\beta}_k} \sum_{t=0}^{k-1} \frac{\beta_t^2}{\tilde{\mu}_{t-1}} 
+ \varsigma 
+ \frac{\bar{\epsilon}_{opt}(k)}{\bar{\beta}_k},
\end{aligned}
\end{equation}
for any $s \in \mathbb{S}$, where $\bar{\beta}_k = \sum_{t=0}^{k-1} \beta_t$, and the cumulative optimization error term is defined as:
\begin{equation}
\begin{aligned}
\bar{\epsilon}_{opt}(k) = \mathbb{E}_{s \sim \kappa^{\pi^*}_q}\bigg[
\underbrace{\mathbb{E}[\epsilon_{opt, k-1}(s)] 
+ \sum_{t=0}^{k-1} \mathbb{E}[\epsilon_{opt, t}(s)]}_{\text{Optimality Gap}} 
+ \underbrace{M_{\tilde{Q}} \sum_{t=0}^{k-1} \mathbb{E}\left[\sqrt{\frac{2\beta_t^2\epsilon_{opt, t-1}(s)}{\tilde{\mu}_{t-1}}}\right]}_{\text{Distance to Exact Policy}} \bigg]
\end{aligned}
\end{equation}

When $\tilde{\mu}_d>0$, if we choose $\beta_k=k+1$, $\lambda_k=\tilde{\mu}_d$, and bound $\epsilon_{opt}$ with $\epsilon$, then
\begin{equation}
\begin{aligned}
\frac{2(1-\gamma)}{k(k+1)} \sum_{t=0}^{k-1} \{(t+1)\mathbb{E}[V^{\hat \pi_t}(q) - V^{\pi^*}(q)]\} + \tilde{\mu}_d \mathbb{E}[\mathcal{D}_q(\pi_k, \pi^*)] \\
\le \frac{2\tilde{\mu}_d \mathcal{D}_q(\pi_0, \pi^*)}{k(k+1)} + \frac{4M_{\tilde{Q}}^2}{\tilde{\mu}_d(k+1)} + \varsigma + \frac{2\epsilon}{k} + \frac{4 M_{\tilde{Q}} \sqrt{2\epsilon}}{\sqrt{\tilde{\mu}_d}(k+1)}.
\end{aligned}
\end{equation}

When $\tilde{\mu}_d=0$, if we choose $\beta_k=k+1$, $\lambda_k=\lambda(k+1)^{3/2}$, and bound $\epsilon_{opt}$ with $\epsilon$, then
\begin{equation}
\begin{aligned}
&\frac{2(1-\gamma)}{k(k+1)} \sum_{t=0}^{k-1} \{(t+1)\mathbb{E}[V^{\hat{\pi}_t}(q) - V^{\pi^*}(q)]\} \\
&\le \frac{2\lambda \mathcal{D}_q(\pi_0, \pi^*)\sqrt{k}}{k+1} + \frac{4M_{\tilde{Q}}^2 \sqrt{k+1}}{\lambda k} + \varsigma 
+ \frac{2\epsilon}{k} + 
\frac{4 M_{\tilde{Q}} \sqrt{2\epsilon}}{\sqrt{\lambda}} \frac{(k+1)^{1/4}}{k}. 
\end{aligned}
\end{equation}
\end{theorem}

A complete proof of Theorem~\ref{theorem:convex} is provided in Appendix~\ref{sec:app:proof1}. Our main theoretical result says the $\beta_k$-weighted average optimality gap converges to $\mathcal{O}(\varsigma)$ as $k \to \infty$. 
Thus a randomized iterate drawn according to the normalized $\beta_k$ weights reaches the stated residual neighborhood. The theorem does not imply that this iterate is exactly globally optimal when $\varsigma>0$ or actor error remains. Exact global convergence follows in the error-free case under the theorem's remaining conditions. Last-iterate guarantees would require an additional argument and are not claimed here.
We mention that $\tilde\mu_d \geq 0$ holds for finite actions with exact or linear function approximation of the advantage function, giving global convergence when $\varsigma=0$ and the actor errors vanish. In the finite state--action setting with exact parameterization and exact action solves, the bound for $\tilde\mu_d>0$ is nearly tight up to a $(1-\gamma)^{-1}$ factor (Appendix~\ref{sec:app:tightness}).

\subsection{Convergence when \texorpdfstring{$\tilde{\mu}_d < 0$}{mu-d < 0}}
We next examine the case of $\tilde{\mu}_d < 0$, which means the cumulative advantage term $\sum_{t=0}^k \beta_t \tilde{\psi}(s,a;\theta_t)$ is non-convex in $a$ with bounded lower curvature. In this case, we will establish a different type of convergence based on the negative advantage function.

\begin{theorem}
\label{theorem:non-convex}
Assume that the number of iterations $k$ is fixed a priori.
Under Assumptions~\ref{assumption:1} and~\ref{assumption:2}, if $\tilde{\mu}_d < 0$ and step sizes are chosen such that $\lambda_t=\lambda = k(k+1)|\tilde{\mu}_d|$ and $\beta_t = t+1$, \(t=0,1,\ldots,k-1\), then for any \(s\in \mathbb{S}\), there exists an iteration index \(\bar k(s)\) such that
\begin{equation}
\begin{aligned}
\label{equ:convergence_non_convex}
 -\frac{(M_Q + M_{\tilde{Q}})^2}{|\tilde{\mu}_d|(k+1)}-\frac{4\epsilon}{\bar k(s)+1}
&\le
- \psi^{\hat{\pi}_{\bar{k}(s)}}(s, \hat{\pi}_{\bar{k}(s)+1}(s))\\
&\le
\frac{2\left[ V^{\pi_0}(s) - V^{\pi^*}(s) \right]}{k+1}  
+ \frac{3(M_Q + M_{\tilde{Q}})^2}{(1-\gamma)|\tilde{\mu}_d|(k+1)} 
+ \frac{8\epsilon(H_k + 1)}{(1-\gamma)(k+1)},
\end{aligned}
\end{equation}

where $H_k=\sum_{j=1}^k\frac{1}{j}$ is the k-th harmonic number, and \(\tilde{\mu}_t = \frac{(t+1)(t+2)}{2}\tilde{\mu}_d + k(k+1)|\tilde{\mu}_d|\), for $t=0,\ldots,k-1$.
\end{theorem}

A complete proof of Theorem~\ref{theorem:non-convex} is provided in Appendix~\ref{sec:app:proof2}. As $k \to \infty$, the error term involving the harmonic number $H_k$ vanishes because the $k$ in the denominator dominates the growth of the numerator. Under the cost convention, improvement corresponds to a negative advantage, and $-\psi^{\hat\pi_{\bar k(s)}}(s,\hat\pi_{\bar k(s)+1}(s))$ measures how much the next policy improves upon the current one. Theorem~\ref{theorem:non-convex} shows that, under our assumptions, for each state $s$ there exists an iteration $\bar k(s)$ at which this advantage improvement decreases at an $O(1/k+\log k/k)$ rate. Thus, the algorithm reaches a stage where an update gives little further advantage-based improvement. In the finite state and action setting, the advantage function is proportional to the gradient~\citep{lanPolicyMirrorDescent2023}. In this well-structured setting, Theorem~\ref{theorem:non-convex} can be viewed as converging to a local solution via gradient-norm minimization. For more general problems, how to characterize policies with a small negative advantage function remains an open question.

\section{AA-PDA: Implementation}

\begin{algorithm}[H]
\caption{Actor-Accelerated PDA Training Loop}
\label{alg:train}
\begin{algorithmic}[1]

\STATE \textbf{Input:} Number of iterations $K$; Networks $V_{\theta_V}$, $\psi^\Sigma_{\theta_\psi}$, $\pi_{\theta_\pi}$; Parameters $\lambda, \sigma_0$
\STATE \textbf{Initialize:} $\Sigma_\beta=0, \pi_0$ (prior policy or zero)

\FOR{$k = 0$ to $K-1$}
    \STATE Update coefficients: $\beta \gets k+1$, $\Sigma_\beta \gets \Sigma_\beta+\beta$, and $\sigma \gets \sigma_0/(k+1)^{0.3}$.

    \STATE Collect $\tau = \{(s_i, a_i, c_i, s'_i)\}_{i=1}^{I}$ with noise $\sigma$, and form batches $\textbf{\textit{s}} = \{s_i\}_{i=1}^I, \textbf{\textit{a}} = \{a_i\}_{i=1}^I$.

    \STATE Compute discounted cost-to-go $G$ and normalized advantages $\tilde{A}$.

    \STATE Update value function: $\theta_V \gets \text{argmin}_{\theta_V'} \mathbb{E} \left[ \|V_{\theta_V'}(\textbf{\textit{s}}) - G\|_2^2 \right]$.

    \STATE Fix old sum advantage values: $\psi^\Sigma_{\text{old}} \gets \psi^\Sigma_{\theta_\psi}(\textbf{\textit{s}}, \textbf{\textit{a}})$.

    \STATE Update sum advantage network:
    $\theta_\psi \gets \text{argmin}_{\theta_\psi'} \mathbb{E} \left[ \| \psi^\Sigma_{\theta_\psi'}(\textbf{\textit{s}}, \textbf{\textit{a}}) - (\frac{\Sigma_\beta-\beta}{\Sigma_\beta}\psi^\Sigma_{\text{old}} + \frac{\beta}{\Sigma_\beta} \tilde{A}) \|_2^2 \right]$.

    \STATE Update actor: $\theta_\pi \gets \text{argmin}_{\theta_\pi'} \mathbb{E} \left[\psi^\Sigma_{\theta_\psi}(\textbf{\textit{s}}, \pi_{\theta_\pi'}(\textbf{\textit{s}})) + \frac{\beta^{3/2} \lambda}{\Sigma_\beta} \|\pi_{\theta_\pi'}(\textbf{\textit{s}}) - \pi_0(\textbf{\textit{s}})\|_2^2 \right]$.
\ENDFOR

\end{algorithmic}
\end{algorithm}

For on-policy algorithms such as PPO, Gaussian policies are typically used with learnable parameters to control the variance of the sampling noise. The degree of exploration is subsequently controlled by entropy coefficients. AA-PDA instead adds Gaussian exploration noise around the deterministic actor with standard deviation $\sigma(k)=\sigma_0/(k+1)^{0.3}$, avoiding the undefined $k=0$ denominator. We set $\beta_k=k+1$ and $\lambda_k=\lambda(k+1)^{3/2}$ by the theorem when $\tilde{\mu}_d=0$. Although the theorem derives this schedule under its stated assumptions, we use it here primarily for its empirical performance rather than under an assumption of global convexity of the learned objective. The base values $\lambda=0.5$ and $\sigma_0=1.3$ were tuned on Hopper-v4 and Walker2d-v4 and then fixed across the benchmark. Sections~\ref{sec:sensitivity} and~\ref{sec:limitations} report the sensitivity evidence and the theory--practice gap.

AA-PDA maintains one more network, sum advantage, in addition to the value network and actor network commonly used for on-policy algorithms. This fixed-size network replaces explicit storage and repeated evaluation of a growing collection of past advantage functions. The actor is then trained against the cumulative sum advantage objective plus the fixed-prox regularizer. PPO uses a GAE estimate from the current rollout and does not explicitly include it in subsequent surrogate objectives. In contrast, the PDA policy update depends on a weighted accumulation of advantage functions across iterations. The sum advantage network approximates this accumulation recursively, using its previous prediction and the current GAE estimate as the regression target. The pseudo-code of the actor-accelerated PDA implementation is shown in Algorithm~\ref{alg:train}. Further discussion on implementation is shown in Appendices~\ref{sec:implementation} and~\ref{sec:app:analysis}.

\section{Experiments}
\subsection{Continuous Control Benchmark}
\begin{figure*}[ht]
  \begin{center}
    \centerline{\includegraphics[width=\textwidth]{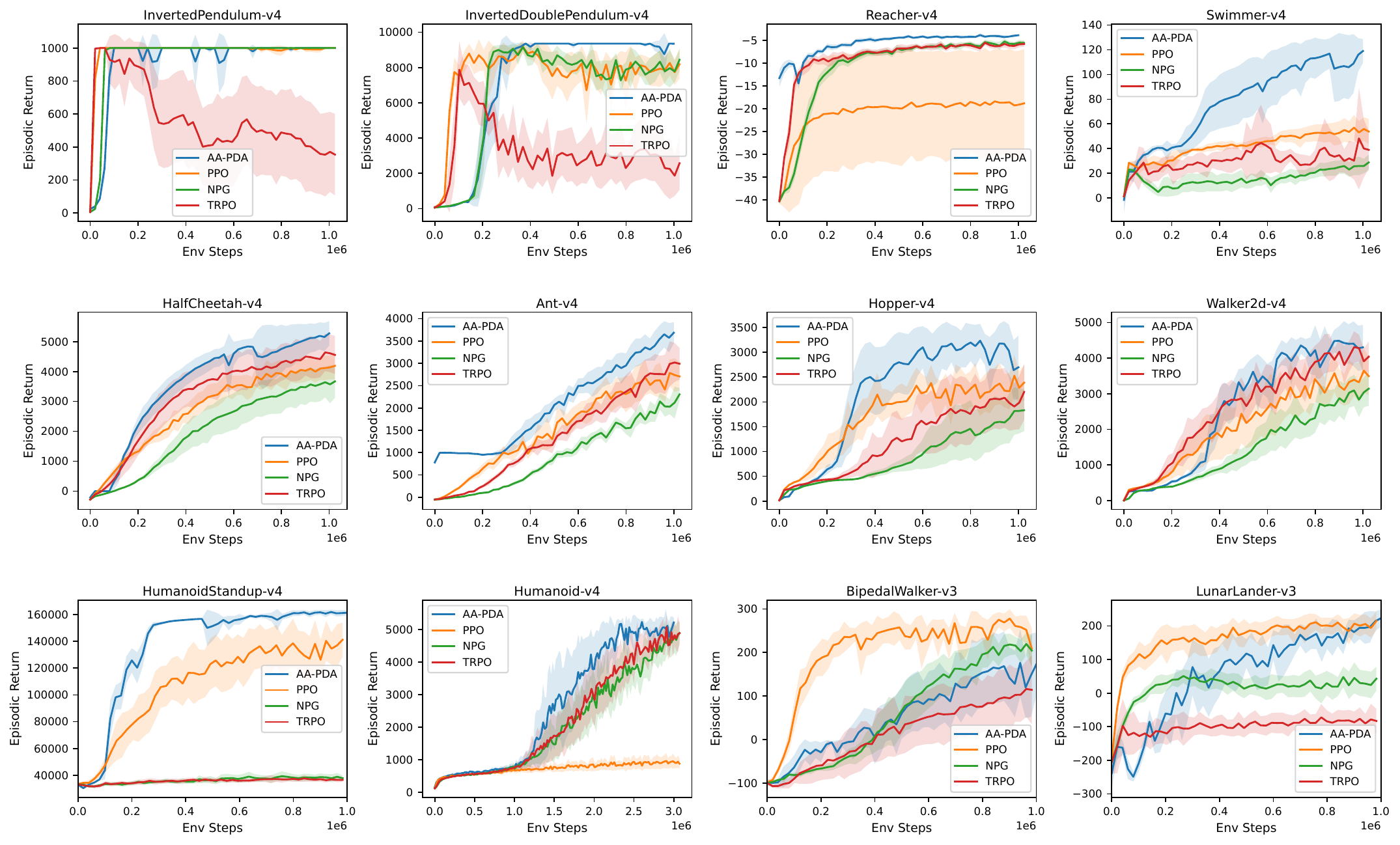}}
    \caption{Performance comparison of AA-PDA, PPO, TRPO, and NPG across MuJoCo-v4 and Box2D environments. Curves show the mean across ten seeds, where each seed is evaluated over ten test episodes. Shaded regions indicate $\pm 2\,\mathrm{SEM}$ (standard error of the mean) across seed-level means.}
    \label{fig:mujoco}
  \end{center}
\end{figure*}

We implemented AA-PDA within the Tianshou~\citep{wengTianshouHighlyModularized2022} framework, an open-source RL library built on PyTorch~\citep{paszkePyTorchImperativeStyle2019} and Gymnasium~\citep{towersGymnasiumStandardInterface2025}. This enables a robust implementation and fair comparisons with other RL algorithms using its well-tuned hyperparameters. Tianshou's PPO implementation achieves better performance in episodic return on MuJoCo~\citep{todorovMuJoCoPhysicsEngine2012} benchmarks compared to the original implementation. For this reason, we adopt Tianshou's PPO as the primary baseline for evaluating our AA-PDA algorithm.

We evaluate AA-PDA on classic continuous control environments, including MuJoCo and Box2D, and compare its performance against widely used on-policy algorithms such as PPO, TRPO, and NPG~\citep{kakadeNaturalPolicyGradient2001}. AA-PDA hyperparameters (Table~\ref{tab:hyperparameters}) are tuned on Hopper and Walker2d and then \textit{kept fixed for all environments}. Hyperparameters for PPO, TRPO, and NPG are from Tianshou and also \textit{kept constant throughout}. Note that Tianshou's official benchmarks report the maximum testing episodic return over 1 million environment steps, while we use the average testing episodic return of the final five epochs, a more conservative and commonly used metric.

AA-PDA outperforms PPO and the other on-policy baselines in most tasks, as shown in Figure~\ref{fig:mujoco} and Table~\ref{tab:benchmark}. In particular, AA-PDA achieves strong performance in locomotion tasks including HalfCheetah, Ant, Walker2d, Hopper, and the Humanoid variants. For the challenging Humanoid variants, AA-PDA has a higher final-window mean than PPO within the stated 1M or 3M horizon using the default parameters.

\subsection{Comparison with Direct PDA}
\label{sec:directpda}
To isolate the effect of the actor, we compare AA-PDA with two direct implementations of PDA that solve the action subproblem whenever an action is queried: PDA-RSGF uses the randomized stochastic gradient-free (RSGF) method~\citep{ghadimi2013stochastic}, and PDA-ACFGM uses the auto-conditioned fast gradient method (ACFGM)~\citep{li2025simple}. All three methods share the same policy-evaluation settings. On Ant, HalfCheetah, Hopper, and Walker2d with ten seeds, AA-PDA reaches an aggregate return comparable to the direct solvers while training about 4 and 8 times faster than PDA-RSGF and PDA-ACFGM, respectively (Table~\ref{tab:directpda}). Learning curves and runtime breakdowns are in Appendix~\ref{sec:app:directpda}.

\begin{table}[ht]
\centering
\caption{AA-PDA versus direct PDA action solvers on four MuJoCo-v4 tasks (ten seeds, 1M environment steps). Return is normalized by the PDA-RSGF final-return reference and averaged over tasks. Runtime is averaged over tasks within seed. Entries are mean $\pm2\,\mathrm{SEM}$ across seeds.}
\label{tab:directpda}
\small
\begin{tabular}{lccc}
\toprule
Method & Normalized return & Runtime per task (s) & Runtime ratio \\
\midrule
AA-PDA & $1.176 \pm 0.117$ & $626.4 \pm 12.2$ & $1.00\times$ \\
PDA-RSGF & $1.026 \pm 0.088$ & $2515.9 \pm 47.3$ & $4.02\times$ \\
PDA-ACFGM & $1.253 \pm 0.092$ & $5004.0 \pm 31.4$ & $7.99\times$ \\
\bottomrule
\end{tabular}
\end{table}

\subsection{Operations Research Benchmark}
\begin{figure}[ht]
  \centering
  \begin{minipage}[c]{0.56\linewidth}
    \centering
    \includegraphics[width=\linewidth]{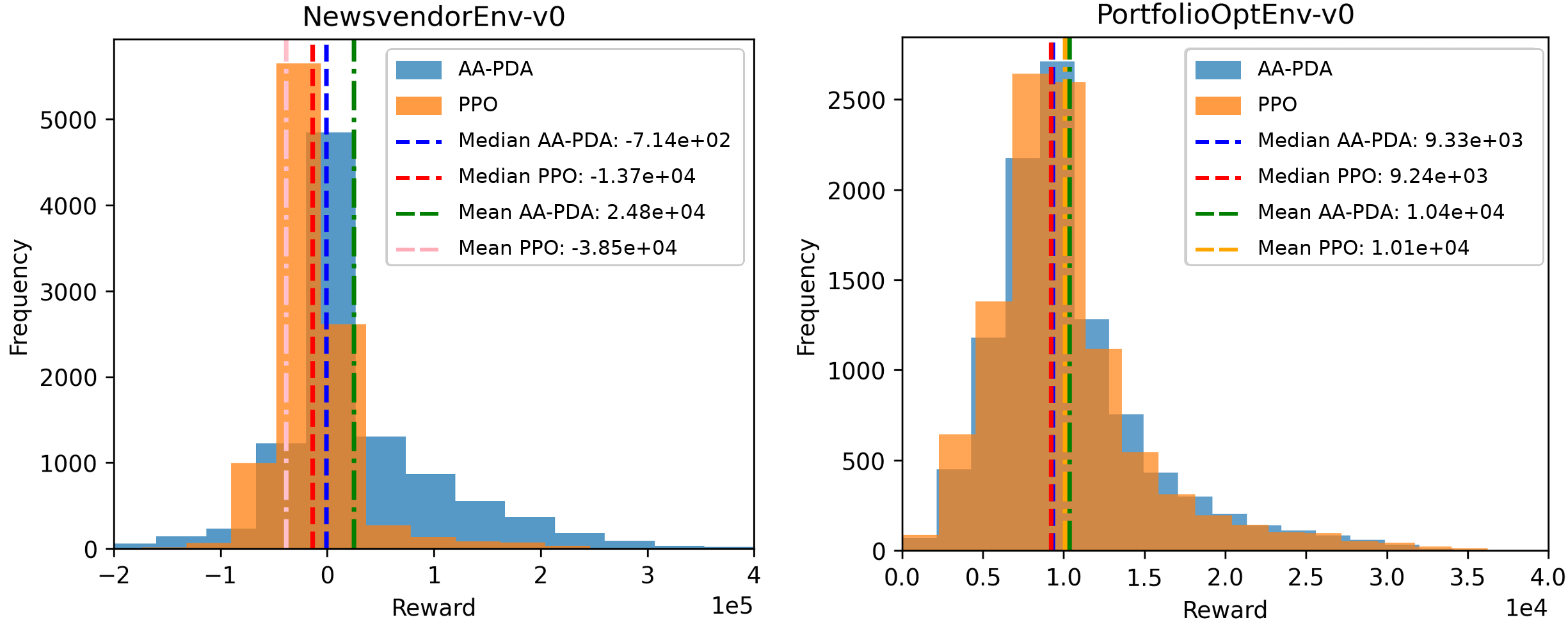}
  \end{minipage}\hfill
  \begin{minipage}[c]{0.40\linewidth}
    \centering
    \includegraphics[width=\linewidth]{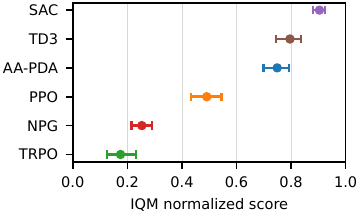}
  \end{minipage}
  \caption{Left: OR-Gym return distributions for AA-PDA and PPO, aggregated over \(10^3\) evaluations from each of ten seeds. Vertical lines mark the mean and median. Newsvendor uses a lower display threshold for readability. Training curves are provided in Appendix~\ref{sec:app:or}. Right: aggregate interquartile mean (IQM) of min--max normalized final returns over 12 MuJoCo-v4 and Box2D environments with ten seeds each. Error bars are 95\% stratified-bootstrap intervals (Appendix~\ref{sec:app:aggregate}).}
  \label{fig:or}
  \label{fig:iqm}
\end{figure}

Beyond continuous control tasks, RL shows strong potential for addressing decision-making problems in operations research (OR). To assess AA-PDA's performance, we evaluate AA-PDA across four OR environments from the OR-Gym benchmark suite~\citep{hubbsORGymReinforcementLearning2020, balajiORLReinforcementLearning2019, glassermanSensitivityAnalysisBaseStock1995, dantzigMultistageStochasticLinear1993}. Our evaluation includes comparisons with on-policy algorithms such as PPO, as well as traditional OR methods. We first consider two classic tasks in stochastic optimization: Newsvendor, a multi-period newsvendor problem with lead times, and PortfolioOpt, a multi-period asset allocation problem (Figure~\ref{fig:or}, left), where we examine the return distribution of the trained agents. In the Newsvendor environment, AA-PDA has higher pooled mean and median values than PPO, where a mean larger than the median indicates that the displayed return distribution is positively skewed. In the PortfolioOpt environment, AA-PDA achieves slightly better overall performance, with its return distribution shifted positively relative to PPO. For InvManagement environments (multi-echelon supply chain inventory management problems with backlog or lost sales), we compare the RL policies against classical OR methods, including Shrinking Horizon Linear Programming (SHLP), Derivative-Free Optimization (DFO), and Mixed-Integer Programming (MIP). These methods are further compared to an Oracle model that computes a theoretical upper bound using knowledge of the actual demand. As shown in Table~\ref{tab:or}, the AA-PDA policy achieves comparable performance to PPO. In addition, AA-PDA achieves returns close to SHLP in both Backlog and Lost Sales settings, with much lower standard deviations than the aforementioned classical OR methods.

\begin{table}[ht]
\centering
\caption{Performance comparison of AA-PDA, PPO, and classical OR methods on the InvManagement environments with Backlog and Lost Sales settings. We report the mean episodic return, standard deviation (Std), and performance ratio relative to the Oracle benchmark (lower is better). Results for the OR methods and the Oracle are taken from the original study~\citep{hubbsORGymReinforcementLearning2020}.}
\label{tab:or}
\tiny
\begin{tabular}{l cccccc cccccc}
\cmidrule[0.8pt](lr){1-7} \cmidrule[0.8pt](lr){8-13}
& \multicolumn{6}{c}{\textbf{Backlog}} & \multicolumn{6}{c}{\textbf{Lost Sales}} \\
\cmidrule(lr){1-7} \cmidrule(lr){8-13}
Metric & AA-PDA & PPO & SHLP & DFO & MIP & Oracle & AA-PDA & PPO & SHLP & DFO & MIP & Oracle \\
\cmidrule(lr){1-7} \cmidrule(lr){8-13}
Mean & 491.6 & 496.4 & 508.0 & 360.9 & 388.0 & 546.8 & 472.3 & 447.4 & 485.4 & 364.3 & 378.5 & 542.7 \\
Std & 11.8 & 6.8 & 28.1 & 39.9 & 30.8 & 30.3 & 10.3 & 16.8 & 29.1 & 33.8 & 26.1 & 29.9 \\
Ratio & 1.1 & 1.1 & 1.1 & 1.5 & 1.4 & 1.0 & 1.1 & 1.2 & 1.1 & 1.5 & 1.4 & 1.0 \\
\cmidrule[0.8pt](lr){1-7} \cmidrule[0.8pt](lr){8-13}
\end{tabular}
\end{table}

\subsection{Comparison with Off-Policy Methods}
\label{sec:compare_offpolicy}
While a direct comparison between on-policy methods and off-policy methods such as Soft Actor-Critic (SAC)~\citep{haarnoja18b2018soft} and Twin Delayed Deep Deterministic Policy Gradient (TD3)~\citep{fujimoto18a2018addressing} is not strictly fair due to differences in data utilization, we nonetheless benchmark AA-PDA against these methods to assess its performance. Per-task and aggregate results are in Appendices~\ref{sec:app:offpolicy} and~\ref{sec:app:aggregate}. In the aggregate comparison (Figure~\ref{fig:iqm}, right), SAC has the strongest IQM, $0.904$ [0.880, 0.924]. AA-PDA has IQM $0.749$ [0.699, 0.794] and TD3 has $0.796$ [0.746, 0.837], with overlapping stratified-bootstrap intervals. TD3 has much higher final-window means on environments such as HalfCheetah-v4, Humanoid-v4, and BipedalWalker-v3, whereas AA-PDA is stronger on some other tasks.
On our hardware, AA-PDA trains roughly 13--48 times faster in wall-clock time than SAC and 20--27 times faster than TD3 (Table~\ref{tab:runtime} in Appendix~\ref{sec:app:runtime}).

\subsection{Sensitivity Study for $\sigma_0$ and $\lambda$}
\label{sec:sensitivity}
AA-PDA has two algorithm-specific hyperparameters, the step size $\lambda$ and the initial exploration noise $\sigma_0$, fixed at $\lambda=0.5$ and $\sigma_0=1.3$ throughout the benchmark. A sweep over both on Hopper, Walker2d, Ant, HalfCheetah, and BipedalWalker (Appendix~\ref{sec:app:sensitivity}) shows a broad region of competitive settings rather than a single tuned configuration. The preferred region is task dependent: the balancing tasks (Hopper, Walker2d, BipedalWalker) favor higher exploration noise, whereas Ant and HalfCheetah favor lower noise and larger step sizes. Constant noise can work well on individual tasks, but no single constant level suits all of them, which motivates the decreasing noise schedule.

\section{Limitations}
\label{sec:limitations}
\label{sec:limitation_lambda}
First, AA-PDA trains three neural networks, a value function, a sum advantage function, and an actor, compared with two for PPO. This overhead is not significant in our benchmarks with small networks ($64\times64$), but may make a difference with larger networks.
Second, the convergence analysis relies on statewise bounds on the actor and evaluation errors, which are needed for the distribution-free guarantees. These errors cannot be computed exactly. The sampled diagnostics in Appendix~\ref{sec:app:diagnostics} provide empirical support for the assumptions rather than a verification. In addition, the schedules for $\beta_k$ and $\lambda_k$ follow the $\tilde\mu_d=0$ case of the analysis, although the implementation does not assume $\tilde\mu_d=0$.
Third, AA-PDA is on-policy and does not reuse replay-buffer data. An off-policy extension would change the interpretation of the sum advantage function and the cumulative objective, and would require extending the convergence analysis.
Finally, our experiments use standard multilayer perceptrons (MLPs) and established baselines. Combining AA-PDA with recent architectural and off-policy advances is left for future work.

\section{Conclusion}
Actor-accelerated PDA makes PDA practical in continuous action spaces by amortizing the per-query action subproblem with a policy network, and our analysis shows how actor and advantage-evaluation errors enter its convergence guarantees. Empirically, it trains about 4--8 times faster than direct PDA solvers with comparable return and outperforms PPO, TRPO, and NPG on most evaluated tasks.

\section*{Acknowledgements and Disclosures}
This work was supported by NSF-Cyber-Physical Systems (CPS)-NIFA (2024-67021-43862) and the NSF AI Institute for Advances in Optimization (2112533). The authors declare no competing interests.

\bibliography{ref}

\newpage
\appendix

\section{Further Discussion on Assumptions}\label{sec:app:assumptions}

\subsection{Stochastic Costs}
\label{sec:app:stochcost}
The observed one-step cost may be stochastic. If the environment returns a random cost \(\hat c(s,a)\), we take \(c(s,a) := \mathbb{E}[\hat c(s,a) \mid s,a]\) and only require \(\hat c(s,a)\) to be an unbiased estimate of \(c(s,a)\) for every fixed state--action pair, so that no additional bias enters policy evaluation.

\subsection{Explanation on $\rho$}
\label{sec:rho}

By deriving convergence results for any distribution $\rho$, our results are \textit{distribution-free}, which leads to some important consequences~\citep{ju2024strongly}.
First, it ensures nonlinear programming methods for solving~\eqref{eq:np} achieve a similar convergence rate as dynamic programming (e.g.,~value iteration) and linear programming methods. Second, it leads to faster performance because the policy optimization does not depend on the unknown stationary distribution $\nu^*$ of the optimal policy $\pi^*$ as seen in~\citep{lanPolicyMirrorDescent2023,liHomotopicPolicyMirror2024}.

\subsection{Tightness of the Convergence Guarantees}
\label{sec:app:tightness}
For general problems with arbitrary bias $\varsigma>0$ and actor-approximation error $\epsilon>0$, there do not exist tractable lower sample-complexity bounds, since the subproblems involve solving non-linear least-squares problems, which are NP-hard. In the simpler finite-state and finite-action setting, exact parameterization and exact action solves are possible, so $\varsigma=0$ and $\epsilon=0$. In this case, the bound for $\tilde\mu_d>0$ in Theorem~\ref{theorem:convex} is nearly tight up to additional dependence on $(1-\gamma)^{-1}$, using the minimax rate in~\citet{azar2013minimax}. We conjecture that the bound for $\tilde\mu_d=0$ in Theorem~\ref{theorem:convex} is also nearly tight, but no lower bounds have been established for this setting yet.

\subsection{Interpreting and Monitoring the Error Bounds}
\label{sec:app:monitor}
The constant $\varsigma$ in Assumption~\ref{assumption:1} bounds the sum of two deterministic error sources, the function-approximation error and the bias of the learned advantage function, evaluated at the action of the executed actor policy $\hat\pi_{k-1}$ and at the optimal-policy action. These terms generally cannot be computed exactly, as they depend on population expectations and on the unknown $\pi^*$. Finding a globally optimal neural-network approximation is itself a non-convex and possibly intractable problem. No separate bound is imposed on the stochastic error $\delta_k^{sto}$. Its effect enters the analysis through the Lipschitz constant $M_{\tilde Q}$.

Appendix~\ref{sec:app:diagnostics} decomposes the sampled advantage error into its stochastic and deterministic parts and estimates a proxy for $\varsigma$ with $\pi^*$ replaced by a best-found action (Figure~\ref{fig:varsigma_proxy}). Figure~\ref{fig:actor_quality} compares the actor output with best-found solutions of the action subproblem, which measures the optimality gap in Assumption~\ref{assumption:2}. In these diagnostics, the actor's normalized gap to the best-found solutions rises early in training and then levels off. The 95th-percentile curve of the $\varsigma$ proxy remains within a finite observed range over 4M steps.

\section{Proof of Theorems}

\label{sec:app:proof}
\subsection{Lemmas}

\begin{lemma}
  \label{lem:1}
  For any feasible policy \(\hat{\pi}_{k+1}(s)\) and \({\pi}_{k+1}(s)\):
\begin{equation*}
    \tilde{\Psi}_k(s, \hat{\pi}_{k+1}(s)) - \epsilon_{opt, k}(s) + \tilde{\mu}_k D(\pi_{k+1}(s), a) \le \tilde{\Psi}_k(s, a)
\end{equation*}
\end{lemma}
\begin{proof}
From the strong convexity of $\tilde{\Psi}_k$ and the generalized optimality condition of $\pi_{k+1}$ with Bregman divergence, we have: \(\tilde{\Psi}_k(s, \pi_{k+1}(s)) + \tilde{\mu}_k D(\pi_{k+1}(s), a) \le \tilde{\Psi}_k(s, a).\) Substituting  $\tilde{\Psi}_k(s, \pi_{k+1}(s))$ with Assumption~\ref{assumption:2} gives the result.
\end{proof}

\subsection{Proof of Theorem~\ref{theorem:convex}}
\label{sec:app:proof1}

\begin{proof}
We first write the recursive form of $\tilde{\Psi}_k$ (which follows directly from its definition) for $k \geq 0$:
\begin{equation*}
    \tilde{\Psi}_k(s, a) = \tilde{\Psi}_{k-1}(s, a) + \beta_k \tilde{\psi}(s, a;\theta_k) + (\lambda_k - \lambda_{k-1}) D(\pi_0(s), a),
\end{equation*}
where, for the index $-1$, we use the convention $\lambda_{-1}:=\lambda_0>0$, $\tilde{\Psi}_{-1}(s,a):=\lambda_0 D(\pi_0(s),a)$, $\hat\pi_0:=\pi_0$, and $\epsilon_{opt,-1}(s):=0$ for all $s,a$. Setting $a=\hat{\pi}_{k+1}(s)$ above and taking note that $\lambda_k - \lambda_{k-1} \ge 0$ and the Bregman divergence is nonnegative, we arrive at:

\begin{equation*}
    \beta_k \tilde{\psi}(s, \hat{\pi}_{k+1}(s); \theta_k) \le \tilde{\Psi}_k(s, \hat{\pi}_{k+1}(s)) - \tilde{\Psi}_{k-1}(s, \hat{\pi}_{k+1}(s)).
\end{equation*}
Using Lemma~\ref{lem:1} with index \(k-1\) evaluated \(a=\hat{\pi}_{k+1}(s)\), we have:
\begin{equation*}
    -\tilde{\Psi}_{k-1}(s, \hat{\pi}_{k+1}(s)) \le -\tilde{\Psi}_{k-1}(s, \hat{\pi}_k(s)) + \epsilon_{opt, k-1}(s)- \tilde{\mu}_{k-1} D(\pi_{k}(s), \hat{\pi}_{k+1}(s)).
\end{equation*}
Combine the two inequalities above, we have:
\begin{equation*}
\beta_k \tilde{\psi}(s, \hat{\pi}_{k+1}(s); \theta_k) \le \tilde{\Psi}_k(s, \hat{\pi}_{k+1}(s)) - \tilde{\Psi}_{k-1}(s, \hat{\pi}_k(s)) + \epsilon_{opt, k-1}(s) - \tilde{\mu}_{k-1} D(\pi_{k}(s), \hat{\pi}_{k+1}(s))
\end{equation*}

Applying this inequality recursively yields a telescopic sum. Since $\tilde{\Psi}_{-1}(s,\hat\pi_0(s)) = \lambda_0 D(\pi_0(s),\pi_0(s)) = 0$, we have:
\begin{equation*}
    \sum_{t=0}^{k-1} \beta_t \tilde{\psi}(s, \hat{\pi}_{t+1}(s); \theta_t) \le \tilde{\Psi}_{k-1}(s, \hat{\pi}_{k}(s)) + \sum_{t=0}^{k-1} \epsilon_{opt, t}(s) - \sum_{t=0}^{k-1} \tilde{\mu}_{t-1} D(\pi_t(s), \hat{\pi}_{t+1}(s)).
\end{equation*}

We now apply Lemma~\ref{lem:1} with index $k-1$ and an arbitrary action $a$  to the inequality derived above to upper bound the \(\tilde{\Psi}_{k-1}(s, \hat{\pi}_{k}(s))\)  term:

\begin{equation*}
\begin{aligned}
\sum_{t=0}^{k-1} \beta_t \tilde{\psi}(s, \hat{\pi}_{t+1}(s); \theta_t) 
\le &
\tilde{\Psi}_{k-1}(s, a) 
+ \epsilon_{opt, k-1}(s) 
- \tilde{\mu}_{k-1} D(\pi_k(s), a) \\
&+\sum_{t=0}^{k-1} \epsilon_{opt, t}(s)
- \sum_{t=0}^{k-1} \tilde{\mu}_{t-1} D({\pi}_t(s), \hat{\pi}_{t+1}(s)).
\end{aligned}
\end{equation*}
Expanding $\tilde{\Psi}_{k-1}(s, a)$ with definition and substituting, we have:
\begin{equation*}
\begin{aligned}
    \sum_{t=0}^{k-1} \beta_t [\tilde{\psi}(s, \hat{\pi}_{t+1}(s); \theta_t)  &- \tilde{\psi}(s, a; \theta_t)] 
    \le
    \lambda_{k-1} D(\pi_0(s), a) - \tilde{\mu}_{k-1} D(\pi_k(s), a) \\
    &+ \sum_{t=0}^{k-1} \epsilon_{opt, t}(s) + \epsilon_{opt, k-1}(s) - \sum_{t=0}^{k-1} \tilde{\mu}_{t-1} D(\pi_t(s), \hat{\pi}_{t+1}(s)).
\end{aligned}
\end{equation*}
For the first term on the lefthand side, use Lipschitz continuity and function approximation error from Assumption~\ref{assumption:1}:
\begin{equation*}
    \tilde{\psi}(s, \hat{\pi}_{t+1}(s);\theta_t) \ge - M_{\tilde{Q}} \|\hat{\pi}_{t+1}(s) - \hat{\pi}_t(s)\| + \delta_t(s, \hat{\pi}_t(s)).
\end{equation*}
For the \(\tilde{\psi}(s, a; \theta_t)\) term on the lefthand side, apply function approximation error again and let $a=\pi^*(s)$:
\begin{equation*}
    -\tilde{\psi}(s, \pi^*(s); \theta_t) = -\psi^{\hat{\pi}_t}(s, \pi^*(s)) - \delta_t(s, \pi^*(s)).
\end{equation*}
Combining these, we have: 
\begin{equation*}
\begin{aligned}
    \sum_{t=0}^{k-1} \beta_t [-\psi^{\hat{\pi}_t}(s, \pi^*(s))] 
    \le &
    \lambda_{k-1} D(\pi_0(s), \pi^*(s)) - \tilde{\mu}_{k-1} D(\pi_k(s), \pi^*(s)) \\
    &+ \sum_{t=0}^{k-1} \epsilon_{opt, t}(s) + \epsilon_{opt, k-1}(s) - \sum_{t=0}^{k-1} \tilde{\mu}_{t-1} D(\pi_t(s), \hat{\pi}_{t+1}(s))\\
    &+ \sum_{t=0}^{k-1} \beta_t [\delta_t(s, \pi^*(s)) - \delta_t(s, \hat\pi_t(s))]  
    + \sum_{t=0}^{k-1} \beta_t M_{\tilde{Q}} \|\hat{\pi}_{t+1}(s) - \hat{\pi}_t(s)\|.
\end{aligned}
\end{equation*}

To simplify the last three summations above, we use the triangle inequality \( \|\hat{\pi}_{t+1}(s) - \hat{\pi}_t(s)\| \le \|\hat{\pi}_{t+1}(s) - \pi_t(s)\| + \|\pi_t(s) - \hat{\pi}_t(s)\|\) in the following derivation:  
\begin{equation*}
\begin{aligned}
&\sum_{t=0}^{k-1} \beta_t M_{\tilde{Q}}
  \|\hat{\pi}_{t+1}(s) - \hat{\pi}_t(s)\|
  - \sum_{t=0}^{k-1} \tilde{\mu}_{t-1}
  D(\pi_t(s), \hat{\pi}_{t+1}(s)) \\
&\quad + \sum_{t=0}^{k-1} \beta_t
  [\delta_t(s, \pi^*(s)) - \delta_t(s, \hat\pi_t(s))] \\
&\le \sum_{t=0}^{k-1} \beta_t M_{\tilde{Q}}
  \|\hat{\pi}_{t+1}(s) - \pi_t(s)\|
  - \sum_{t=0}^{k-1} \tilde{\mu}_{t-1}
  D(\pi_t(s), \hat{\pi}_{t+1}(s)) \\
&\quad + \sum_{t=0}^{k-1} \beta_t M_{\tilde{Q}}
  \|\pi_t(s) - \hat{\pi}_t(s)\|
  + \varsigma \sum_{t=0}^{k-1} \beta_t \\
&\le \sum_{t=0}^{k-1}
  \frac{M_{\tilde{Q}}^2 \beta_t^2}{2 \tilde{\mu}_{t-1}}
  + \sum_{t=0}^{k-1} \beta_t M_{\tilde{Q}}
  \sqrt{\frac{2\epsilon_{opt, t-1}(s)}{\tilde{\mu}_{t-1}}}
  + \varsigma \sum_{t=0}^{k-1} \beta_t,
\end{aligned}
\end{equation*}
where the first inequality uses Assumption~\ref{assumption:1}. This step holds in expectation, which is taken below. Note that $\hat\pi_t$ is $F_{t-1}$-measurable, so $\mathbb{E}[\delta^{sto}_t(s,\hat\pi_t(s))\mid F_{t-1}]=0$, and likewise at the deterministic action $\pi^*(s)$. Hence $\mathbb{E}\sum_{t=0}^{k-1}\beta_t[\delta_t(s,\pi^*(s))-\delta_t(s,\hat\pi_t(s))]\le\varsigma\sum_{t=0}^{k-1}\beta_t$ by Assumption~\ref{assumption:1}. Meanwhile, the first summation in the second inequality uses strong convexity of the Bregman divergence $D(x, y) \ge \frac{1}{2}\|x-y\|^2$ and the inequality $bx - \frac{a}{2}x^2 \le \frac{b^2}{2a}$, and the second summation follows from Assumption~\ref{assumption:2}.

Now fix any $q \in \mathbb{S}$. Combining the last two intermediate results, taking expectations first w.r.t. the noise $\xi$ and then the state $s \sim \kappa_q^{\pi^*}$, applying~\eqref{eq:performance_difference} with $\pi' = \pi^*$, and dividing through by $\bar{\beta}_k := \sum_{t=0}^{k-1} \beta_t$, we have

\begin{equation*}
\begin{aligned}
(1-\gamma) \frac{1}{\bar{\beta}_k} \sum_{t=0}^{k-1} \beta_t \mathbb{E}[V^{\hat{\pi}_t}(q) - V^{\pi^*}(q)] 
+ \frac{\tilde{\mu}_{k-1}}{\bar{\beta}_k} \mathbb{E}[\mathcal{D}_q(\pi_k, \pi^*)] 
\\
\le 
\frac{\lambda_{k-1}}{\bar{\beta}_k} \mathcal{D}_q(\pi_0, \pi^*) 
+ \frac{M_{\tilde{Q}}^2}{2\bar{\beta}_k} \sum_{t=0}^{k-1} \frac{\beta_t^2}{\tilde{\mu}_{t-1}} 
+ \varsigma 
+ \frac{\bar{\epsilon}_{opt}(k)}{\bar{\beta}_k},
\end{aligned}
\end{equation*}

where the cumulative optimization error term is defined as: 
\begin{equation*}
\begin{aligned}
\bar{\epsilon}_{opt}(k) = \mathbb{E}_{s \sim \kappa^{\pi^*}_q}\bigg[
\underbrace{\mathbb{E}[\epsilon_{opt, k-1}(s)] 
+ \sum_{t=0}^{k-1} \mathbb{E}[\epsilon_{opt, t}(s)]}_{\text{Optimality Gap}} 
+ \underbrace{M_{\tilde{Q}} \sum_{t=0}^{k-1} \mathbb{E}\left[\sqrt{\frac{2\beta_t^2\epsilon_{opt, t-1}(s)}{\tilde{\mu}_{t-1}}}\right]}_{\text{Distance to Exact Policy}} \bigg]
\end{aligned}
\end{equation*}

When $\tilde{\mu}_d>0$, if we choose $\beta_k=k+1$, $\lambda_k=\tilde{\mu}_d$, and use the approximation error bound $0\le \epsilon_{opt, k}(s)\le\epsilon$, then $\tilde{\mu}_{t-1}  = \tilde{\mu}_d \left( 1 + \frac{t(t+1)}{2} \right)$, and the bound follows after simplifying terms w.r.t.~$k$.

When $\tilde{\mu}_d=0$, if we choose $\beta_k=k+1$, $\lambda_k=\lambda(k+1)^{3/2}$, and use the approximation error bound  $0\le \epsilon_{opt, k}(s)\le\epsilon$, and the integral bound ($\sum_{x=1}^k \sqrt{x} \le \int_0^{k+1} \sqrt{x} dx$), then we get the result after simplifying terms w.r.t.~$k$.
\end{proof}

\subsection{Proof of Theorem~\ref{theorem:non-convex}}
\label{sec:app:proof2}

\begin{proof}
Similar to the proof  for Theorem~\ref{theorem:convex}, we apply Lemma~\ref{lem:1} for $\tilde{\Psi}_{k-1}(s, \hat{\pi}_{k+1}(s))$:
\begin{equation*}
    \beta_k \tilde{\psi}(s, \hat{\pi}_{k+1}(s); \theta_k) \le \tilde{\Psi}_k(s, \hat{\pi}_{k+1}(s)) - \tilde{\Psi}_{k-1}(s, \hat{\pi}_k(s)) + \epsilon_{\text{opt}, k-1}(s) - \tilde{\mu}_{k-1} D(\pi_{k}(s), \hat{\pi}_{k+1}(s)).
\end{equation*}
Apply Lemma~\ref{lem:1} again for $\tilde{\Psi}_k(s, \hat{\pi}_{k+1}(s))$ at step $k$ evaluated at $a=\hat{\pi}_{k}(s)$:
\begin{equation*}
    \tilde{\Psi}_k(s, \hat{\pi}_{k+1}(s)) + \tilde{\mu}_k D(\pi_{k+1}(s), \hat{\pi}_{k}(s)) \le \tilde{\Psi}_k(s, \hat{\pi}_{k}(s)) + \epsilon_{\text{opt}, k}(s).
\end{equation*}
Substituting  $\tilde{\Psi}_k(s, \hat{\pi}_{k+1}(s))$ and using the recursive definition $\tilde{\Psi}_k(s, \hat{\pi}_{k}(s)) = \tilde{\Psi}_{k-1}(s, \hat{\pi}_{k}(s)) + \beta_k \tilde{\psi}(s, \hat{\pi}_{k}(s);\theta_k) + (\lambda_k - \lambda_{k-1}) D(\pi_0(s), \hat{\pi}_{k}(s))$, we obtain:
\begin{equation*}
\begin{aligned}
    \beta_k \tilde{\psi}(s, \hat{\pi}_{k+1}(s); \theta_k) \le &
    - \tilde{\mu}_k D(\pi_{k+1}(s), \hat{\pi}_{k}(s)) - \tilde{\mu}_{k-1} D(\pi_{k}(s), \hat{\pi}_{k+1}(s)) \\
    &+ \beta_k \tilde{\psi}(s, \hat{\pi}_{k}(s);\theta_k) 
    + (\lambda_k - \lambda_{k-1}) D(\pi_0(s), \hat{\pi}_{k}(s)) \\
    &+ \epsilon_{\text{opt}, k}(s) + \epsilon_{\text{opt}, k-1}(s).
\end{aligned}
\end{equation*}
Using the function approximation error from Assumption~\ref{assumption:1} and noting that $\psi^{\hat\pi_k}(s, \hat\pi_k(s))=0$, we have $\tilde{\psi}(s, \hat\pi_k(s); \theta_k) = \delta_k(s, \hat\pi_k(s))$. Rearranging terms yields:
\begin{equation*}
\begin{aligned}
    \psi^{\hat\pi_k}(s, \hat\pi_{k+1}(s)) 
    &- \frac{\lambda_k - \lambda_{k-1}}{\beta_k}D(\pi_0(s), \hat{\pi}_{k}(s))
    - \frac{1}{\beta_k}[\epsilon_{\text{opt}, k}(s) + \epsilon_{\text{opt}, k-1}(s)] \\
    &- [\delta_k(s, \hat\pi_k(s)) - \delta_k(s, \hat\pi_{k+1}(s))] \\
    &\le 
    - \frac{\tilde{\mu}_k}{\beta_k} D(\pi_{k+1}(s), \hat{\pi}_{k}(s)) 
    - \frac{\tilde{\mu}_{k-1}}{\beta_k} D(\pi_{k}(s), \hat{\pi}_{k+1}(s)) 
    \le 0.
\end{aligned}
\end{equation*}
Using the Lipschitz property of $\delta_k$, we bound the terms involving \(\delta_k\) by $\delta_k(s, \hat\pi_k(s)) - \delta_k(s, \hat\pi_{k+1}(s)) \le (M_Q + M_{\tilde{Q}}) \|\hat{\pi}_{k+1}(s) - \hat{\pi}_k(s)\|$.
From the 1-strong convexity of Bregman divergence, we have:
\begin{equation*}
\begin{aligned}
&-\frac{\tilde{\mu}_k}{\beta_k} D(\pi_{k+1}(s), \hat{\pi}_{k}(s)) -\frac{\tilde{\mu}_{k-1}}{\beta_k} D(\pi_{k}(s), \hat{\pi}_{k+1}(s)) 
\\
&\le -\frac{\tilde{\mu}_k}{2\beta_k}
  \|\pi_{k+1}(s)-\hat{\pi}_{k}(s)\|^2
  -\frac{\tilde{\mu}_{k-1}}{2\beta_k}
  \|\pi_{k}(s)-\hat{\pi}_{k+1}(s)\|^2 \\
&\le \frac{\tilde{\mu}_k}{2\beta_k}
  \|\pi_{k+1}(s)-\hat{\pi}_{k+1}(s)\|^2
  +\frac{\tilde{\mu}_{k-1}}{2\beta_k}
  \|\pi_k(s)-\hat{\pi}_k(s)\|^2 \\
&\quad - \frac{\tilde{\mu}_k+\tilde{\mu}_{k-1}}{4\beta_k}
  \|\hat{\pi}_{k+1}(s)-\hat{\pi}_k(s)\|^2.
\end{aligned}
\end{equation*}
The second inequality comes from expanding the policy difference terms using the triangle inequality and Young's inequality to derive $\|a-b\|^2 \geq -\|a\|^2 + \frac{1}{2}\|b\|^2$.

Combining with the Lipschitz bound and using Young's inequality $ax - bx^2 \le a^2 / (4b)$ on the terms involving $\|\hat{\pi}_{k+1}-\hat{\pi}_k\|$:
\begin{equation*}
\begin{aligned}
&(M_Q + M_{\tilde{Q}})
  \|\hat{\pi}_{k+1}(s) - \hat{\pi}_k(s)\|
  - \frac{\tilde{\mu}_k+\tilde{\mu}_{k-1}}{4\beta_k}
  \|\hat{\pi}_{k+1}(s)-\hat{\pi}_k(s)\|^2 \\
&\qquad \le
  \frac{\beta_k(M_Q + M_{\tilde{Q}})^2}
       {\tilde{\mu}_k+\tilde{\mu}_{k-1}}.
\end{aligned}
\end{equation*}
From the 1-strong convexity of Bregman divergence and the Assumption~\ref{assumption:2}, we have:
\begin{equation*}
    \tilde{\mu}_k \|\pi_{k+1}(s)-\hat{\pi}_{k+1}(s)\|^2 \le 2\epsilon_{\text{opt}, k}(s).
\end{equation*}

Since $\lambda_t=\lambda$ for all $t\ge -1$ under our convention, $\lambda_k-\lambda_{k-1}=0$ for all $k\ge 0$, and the term $\frac{\lambda_k-\lambda_{k-1}}{\beta_k}D(\pi_0(s),\hat\pi_k(s))$ vanishes. Combine the relationships to result in the bound for the advantage:
\begin{equation}
\label{equ:adv_bound}
    0 \le
    - \psi^{\hat\pi_k}(s, \hat\pi_{k+1}(s))
    + \frac{2}{\beta_k}[\epsilon_{\text{opt}, k}(s) + \epsilon_{\text{opt}, k-1}(s)] 
    + \frac{\beta_k(M_Q + M_{\tilde{Q}})^2}{\tilde{\mu}_k+\tilde{\mu}_{k-1}}.
\end{equation}

Define constants $C_k\ge0$ and $E_k\ge0$ as:
\begin{equation*}
    C_k := \frac{\beta_k(M_Q + M_{\tilde{Q}})^2}{\tilde{\mu}_k+\tilde{\mu}_{k-1}},
\end{equation*}
\begin{equation*}
    E_k := \frac{4}{\beta_k}\epsilon \ge\frac{2}{\beta_k}[\epsilon_{\text{opt}, k}(s) + \epsilon_{\text{opt}, k-1}(s)].
\end{equation*}
Using the performance difference lemma (Equation~\ref{eq:performance_difference}), since $\kappa_{s}^{\hat\pi_{k+1}}$ is a probability measure, we can add and subtract a constant term in and outside the integral:
\begin{equation}
\begin{aligned}
\label{equ:performance_diff_1}
V^{\hat\pi_{k+1}}(s)-V^{\hat\pi_{k}}(s) 
&\le
\frac{1}{1-\gamma}\int \left[
\psi^{\hat\pi_{k}}(q,\hat\pi_{k+1}(q))
- C_k - E_k
\right]\kappa_{s}^{\hat\pi_{k+1}}(dq) 
+  \frac{1}{1-\gamma} (C_k + E_k)
\\
&\le \frac{1}{1-\gamma} \left[
\psi^{\hat\pi_{k}}(s,\hat\pi_{k+1}(s))
- C_k - E_k
\right]\kappa_{s}^{\hat\pi_{k+1}}(\{s\}) 
+  \frac{1}{1-\gamma} (C_k + E_k)
\\
&\le \psi^{\hat\pi_{k}}(s,\hat\pi_{k+1}(s)) + \frac{\gamma}{1-\gamma} (C_k + E_k),
\end{aligned}
\end{equation}
where the second inequality follows the fact $\kappa_s^{\hat\pi_{k+1}}(\{s\}) \ge 1-\gamma$.

We define $S_k := \sum_{t=0}^{k-1} \beta_t [V^{\hat{\pi}_t}(s) - V^{\hat{\pi}_{t+1}}(s)]$. With $\beta_t = t+1$, 

\begin{equation*}
\begin{aligned}
S_k &= \sum_{t=0}^{k-1} (t+1) [V^{\hat{\pi}_t}(s) - V^{\hat{\pi}_{t+1}}(s)] 
= \sum_{t=0}^{k-1} V^{\hat{\pi}_t}(s) - k V^{\hat{\pi}_k}(s)
\le \sum_{t=0}^{k-1} V^{\hat{\pi}_t}(s) - k V^{\pi^*}(s).
\end{aligned}
\end{equation*}

From Equation~\ref{equ:adv_bound} and \ref{equ:performance_diff_1}, we have:
\begin{equation*}
V^{\hat{\pi}_{t+1}}(s) - V^{\hat{\pi}_t}(s) \le \frac{1}{1-\gamma}(C_t + E_t).
\end{equation*}
Substituting the telescopic sum, we have:
\begin{equation*}
V^{\hat{\pi}_t}(s) \le V^{\hat\pi_0}(s) + \sum_{j=0}^{t-1} \frac{1}{1-\gamma}(C_j + E_j)
\end{equation*}
Substituting $V^{\hat{\pi}_t}(s)$ in $S_k$, we have 

\begin{equation*}
\begin{aligned}
S_k &\le \sum_{t=0}^{k-1} \left( V^{\hat\pi_0}(s) + \sum_{j=0}^{t-1} \frac{1}{1-\gamma}(C_j + E_j)\right) - k V^{\pi^*}(s)\\
 &\le kV^{\hat\pi_0}(s) + \frac{1}{1-\gamma} \sum_{t=0}^{k-1} \sum_{j=0}^{t-1} (C_j + E_j) - k V^{\pi^*}(s)
\end{aligned}
\end{equation*}

Using the relationship  \(\sum_{t=0}^{k-1} \frac{\beta_t^2}{\tilde{\mu}_t+\tilde{\mu}_{t-1}} \le \sum_{t=0}^{k-1} \frac{(t+1)^2}{|\tilde{\mu}_d|(k+1)k} \le\frac{k}{|\tilde{\mu}_d|}\), we have:
\begin{equation*}
    S_k \le k \left[ V^{\pi_0}(s) - V^{\pi^*}(s) + \frac{(M_Q + M_{\tilde{Q}})^2}{2(1-\gamma)|\tilde{\mu}_d|} + \frac{4\epsilon H_k}{1-\gamma} \right],
\end{equation*}
where $H_k = \sum_{j=0}^{k-1} \frac{1}{j+1}$ is the $k$-th harmonic number.

Rearranging Equation~\ref{equ:performance_diff_1}, we have:
\begin{equation*}
    0 \le -\beta_k \left[ \psi^{\hat{\pi}_k}(s, \hat{\pi}_{k+1}(s)) - (C_k + E_k) \right] \le \beta_k \left[ V^{\hat{\pi}_{k}}(s) - V^{\hat{\pi}_{k+1}}(s) \right] + \frac{\beta_k}{1-\gamma} (C_k + E_k).
\end{equation*}
Summing the inequality above from \(t=0\) to \(k-1\) and use the inequality for $S_k$:
\begin{equation*}
\begin{aligned}
    0 \le \sum_{t=0}^{k-1} -\beta_t \left[ \psi^{\hat{\pi}_t}(s, \hat{\pi}_{t+1}(s)) - (C_t + E_t) \right] 
    \le &
    k \left[ V^{\pi_0}(s) - V^{\pi^*}(s) \right] \\
    &+ \frac{3k(M_Q + M_{\tilde{Q}})^2}{2(1-\gamma)|\tilde{\mu}_d|} + \frac{4\epsilon}{1-\gamma}(k H_k + k).
\end{aligned}
\end{equation*}

Choosing index $\bar{k}(s)$ that minimizes $-\psi^{\hat{\pi}_t}(s, \hat{\pi}_{t+1}(s)) + C_t + E_t$ in the summation, and using $\beta_{\bar{k}(s)} = \bar{k}(s) + 1$, we derive the upper bound on the advantage for Theorem~\ref{theorem:non-convex}. The lower bound follows from Equation~\ref{equ:adv_bound} at $\bar k(s)$.

To derive a bound for policy improvement, we return to Equation~\ref{equ:adv_bound} without substituting $\|\hat{\pi}_{k+1}(s) - \hat{\pi}_k(s)\|$, we split the term using the triangle inequality and combine with the advantage bound:
\begin{equation*}
\frac{\tilde{\mu}_k+\tilde{\mu}_{k-1}}{8\beta_k}  \|\hat{\pi}_{k+1}(s)-\hat{\pi}_k(s)\|^2  
\le
- \psi^{\hat\pi_k}(s, \hat\pi_{k+1}(s)) 
+ \frac{4\epsilon}{\beta_k}
+ \frac{2\beta_k(M_Q + M_{\tilde{Q}})^2}{\tilde{\mu}_k+\tilde{\mu}_{k-1}} .
\end{equation*}
Combining with Equation~\ref{equ:convergence_non_convex} and the fact that \(0\le
\frac{\beta_{\bar{k}(s)}}{\tilde{\mu}_{\bar{k}(s)}+\tilde{\mu}_{\bar{k}(s)-1}} \le \frac{1}{|\tilde{\mu}_d|(k+1)}\), we have:
\begin{equation*}
\begin{aligned}
\|\hat{\pi}_{\bar{k}(s)+1}(s)-\hat{\pi}_{\bar{k}(s)}(s)\|^2   
 &\le
\frac{8}{|\tilde{\mu}_d|(k+1)}
\bigg[
\frac{2\left[ V^{\pi_0}(s) - V^{\pi^*}(s) \right]}{k+1}  \\
&+ \frac{3(M_Q + M_{\tilde{Q}})^2}{(1-\gamma)|\tilde{\mu}_d|(k+1)} 
+ \frac{2(M_Q + M_{\tilde{Q}})^2}{|\tilde{\mu}_d|(k+1)} \\
&+ \frac{8\epsilon(H_k + 1)}{(1-\gamma)(k+1)}
+ \frac{4\epsilon}{{\bar{k}(s)+1}}
\bigg] 
\end{aligned}
\end{equation*}

\end{proof}

\section{Implementation Details}
\label{sec:implementation}

\subsection{Sum Advantage Estimation}
To maintain numerical stability during training and avoid performance degradation when dealing with large values of the cumulative summation term in the PDA objective in Algorithm~\ref{alg:pda}, we implement a scaled version of the objective function. The scaling factors are chosen to ensure that the weighting remains consistent with the original theoretical framework. In addition, we use a recursive scheme to store the weighted sum of advantages. Specifically, the accumulated advantage is updated as
\begin{equation*}
\begin{aligned}
\psi^{\Sigma}_{\theta_\psi}(s,a) \approx \left(1-\frac{\beta_k}{\sum_{i=0}^k\beta_i}\right)\psi^{\Sigma}_{\text{old}}(s,a) 
+ \frac{\beta_k}{\sum_{i=0}^k\beta_i} \tilde{\psi}(s,a; \theta_k).
\end{aligned}
\end{equation*}

During training, the previous sum advantage prediction is detached from the computational graph before constructing the new regression target, preventing gradients from propagating through earlier advantage models. The recursion therefore compresses the normalized weighted history into a fixed-size network rather than retaining and re-evaluating all previous models.

With this representation, the action can be obtained by solving the following optimization problem:

\begin{equation*}
\begin{aligned}
&\pi_{k+1}(s) = \arg\min_{a \in \mathcal{A}}{\tilde\Psi'}(s,a) 
=\arg\min_{a \in \mathcal{A}} \left[ \psi^{\Sigma}_{\theta_\psi}(s,a) + \frac{\lambda_k}{\sum_{i=0}^k\beta_i}  D(\pi_0(s),a) \right],
\end{aligned}
\end{equation*}

where \(\tilde\Psi'\) represents the scaled objective with parameterized accumulated advantage.

For problems with continuous action space, we choose the Euclidean norm as the distance generating function, hence the corresponding Bregman divergence term for PDA is  $\frac{1}{2}\|a - \pi_0(s)\|_2^2$. The prox-center $\pi_0(s)$ can be a random initial policy, a pretrained policy, or simply set to zero. The approximate solution of the optimization subproblem can then be obtained with standard optimizers, with gradients obtained through backpropagation of action $a$ on ${\tilde\Psi'}$. 

We fix the schedules before training, using \(\beta_k=k+1\) and \(\lambda_k=\lambda(k+1)^{3/2}\). This follows the theoretical $\tilde\mu_d=0$ rule as an empirically selected implementation choice and leaves $\lambda$ as a tunable hyperparameter.

\subsection{Summary of Notation}
\label{sec:app:notation}

For clarity, the notation below summarizes the key symbols used. The weights $\beta_k$ determine how advantage information is accumulated, $\lambda_k$ controls regularization around the fixed prox-center, $\theta_k$ parameterizes the current approximate advantage, $\theta_\psi$ parameterizes the sum advantage network, and $\theta_\pi$ parameterizes the actor. The exact action $\pi_{k+1}(s)$ minimizes the cumulative objective, whereas $\hat\pi_{k+1}(s;\theta_k^\pi)$ is the actor approximation used by AA-PDA.

\begin{itemize}
    \item     The cumulative regularized objective $\tilde{\Psi}_k(s, a) := \sum_{t=0}^k \beta_t \tilde{\psi}(s, a; \theta_t) + \lambda_k D(\pi_0(s), a)$.
    \item     The optimization subproblem: $\pi_{k+1}(s) = \arg\min_{a \in \mathcal{A}} \tilde{\Psi}_k(s, a)$.
    \item $\pi_{k+1}(s)$ denotes the exact minimizer of the cumulative objective $\tilde{\Psi}_k$.
    \item $\psi^{\pi}(s, a)$: the true advantage function under policy $\pi$. 
    \item $\hat{\pi}_{k+1}(s; \theta^\pi_k)$: the approximate solution produced by the actor network, i.e., $\pi_{\theta_\pi}$ in Algorithm~\ref{alg:train} after iteration $k$.
    \item $\tilde{\psi}(s, a; \theta_k)$: the approximate advantage function parameterized by $\theta_k$.
 
    \item The sum advantage network estimates the scaled weighted sum advantage:
\[
\psi^{\Sigma}_{\theta_\psi}(s,a)
\approx \left(1-\frac{\beta_k}{\sum_{i=0}^k\beta_i}\right)
\psi^{\Sigma}_{\text{old}}(s,a)
+ \frac{\beta_k}{\sum_{i=0}^k\beta_i} \tilde{\psi}(s,a; \theta_k).
\]
    \item The scaled cumulative regularized objective:
\[
\tilde\Psi'(s,a) = \psi^{\Sigma}_{\theta_\psi}(s,a)
+ \frac{\lambda_k}{\sum_{i=0}^k\beta_i} D(\pi_0(s),a).
\]

\end{itemize}

\newpage
\subsection{Algorithm Hyperparameters}
\label{sec:app:hyperparams}
\begin{table}[H]
\centering
\caption{Hyperparameters for the learned-policy benchmarks and the direct-PDA action solvers. The direct solvers share the AA-PDA policy-evaluation settings and differ only in the per-query action optimization settings listed in the lower panel. PDA-RSGF and PDA-ACFGM solve each action subproblem directly with the randomized stochastic gradient-free (RSGF) method and the auto-conditioned fast gradient method (ACFGM), respectively (Appendix~\ref{sec:app:directpda}).}
\label{tab:hyperparameters}
\footnotesize
\begin{tabular}{lcccccc}
\toprule
\textbf{Hyperparameter} & \textbf{AA-PDA} & \textbf{NPG} & \textbf{PPO} & \textbf{TRPO}  & \textbf{SAC} & \textbf{TD3}\\
\midrule
Hidden sizes & $[64, 64]$ & $[64, 64]$ & $[64, 64]$ & $[64, 64]$  & $[256, 256]$& $[256, 256]$\\
 Activation & Tanh& Tanh& Tanh& Tanh & ReLU& ReLU\\
Learning rate & $1 \times 10^{-3}$ & $1 \times 10^{-3}$ & $3 \times 10^{-4}$ & $1 \times 10^{-3}$  & $1 \times 10^{-3}$  & $3 \times 10^{-4}$ \\
Batch size & 1000 & 1024& 64 & 1024 & 256& 256\\
Discount factor& 0.99 & 0.99 & 0.99 & 0.99  & 0.99& 0.99\\
GAE $\lambda$ & 0.95 & 0.95 & 0.95 & 0.95  & ---  & ---  \\
 TD n-step& ---  & ---  & ---  & ---  & 1&1\\
Learning rate decay & None & True & True & True  & None& None\\
 Max gradient norm & 0.1 & --- & 0.5&---  & ---  & ---  \\
\midrule
\multicolumn{5}{l}{\textit{Algorithm-Specific Hyperparameters}} & & \\
\midrule
\multicolumn{5}{l}{\textit{AA-PDA}} & & \\
Regularization step size& 0.5 & --- & --- & ---  & --- & ---  \\
Action noise & 1.3 & --- & --- & ---  & --- & ---  \\
\midrule
\multicolumn{5}{l}{\textit{NPG}} & & \\
Actor step size& --- & 0.1 & --- & ---  & --- & ---  \\
Critic optimization iterations & --- & 20 & --- & ---  & --- & ---  \\
\midrule
\multicolumn{5}{l}{\textit{PPO}} & & \\
Clipping parameter& --- & --- & 0.2 & ---  & --- & ---  \\
Value function coefficient & --- & --- & 0.25 & ---  & --- & ---  \\
Entropy coefficient & --- & --- & 0.0 & ---  & --- & ---  \\
Value clipping & --- & --- & None& ---  & --- & ---  \\
Dual clipping& --- & --- & None & ---  & ---  & ---  \\
\midrule
\multicolumn{5}{l}{\textit{TRPO}} & & \\
Max KL divergence & --- & --- & --- & 0.01  & --- & ---  \\
Backtrack coefficient & --- & --- & --- & 0.8  & --- & ---  \\
Max backtracks & --- & --- & --- & 10  & --- & ---  \\
Critic optimization iterations & --- & --- & --- & 20  & --- & --- \\
\midrule
\multicolumn{5}{l}{\textit{SAC}} & & \\
Soft update parameter& --- & --- & --- & --- & 0.005& ---  \\
Entropy coefficient& --- & --- & --- & --- & 0.2& ---  \\
Warmup steps& --- & --- & --- & --- & 10000& ---  \\
\midrule
\multicolumn{5}{l}{\textit{TD3}} & & \\
Soft update parameter& --- & --- & --- & --- & --- & 0.005\\
Exploration noise& --- & --- & --- & --- & --- & 0.1\\
Policy noise& --- & --- & --- & --- & --- & 0.2\\
Noise clip& --- & --- & --- & --- & --- & 0.5\\
Actor update frequency& --- & --- & --- & --- & --- &2\\
Warmup steps& --- & --- & --- & --- & --- &25000\\
\bottomrule
\end{tabular}

\vspace{10ex}
\begin{tabular}{lcc}
\toprule
\textbf{Direct-PDA action-solver setting} & \textbf{PDA-RSGF} & \textbf{PDA-ACFGM}\\
\midrule
Action iterations per query & 75 & 75\\
Smoothing radius $\mu$ & $10^{-4}$ & ---\\
Action learning rate & $10^{-2}$ & ---\\
Momentum / Nesterov & 0.9 / True & ---\\
ACFGM $\beta$ & --- & 0.3\\
\bottomrule
\end{tabular}
\end{table}

\newpage
\subsection{Sensitivity to $\sigma_0$ and $\lambda$}
\label{sec:app:sensitivity}
\begin{figure}[ht]
  \begin{center}
    \centerline{\includegraphics[width=0.7\textwidth]{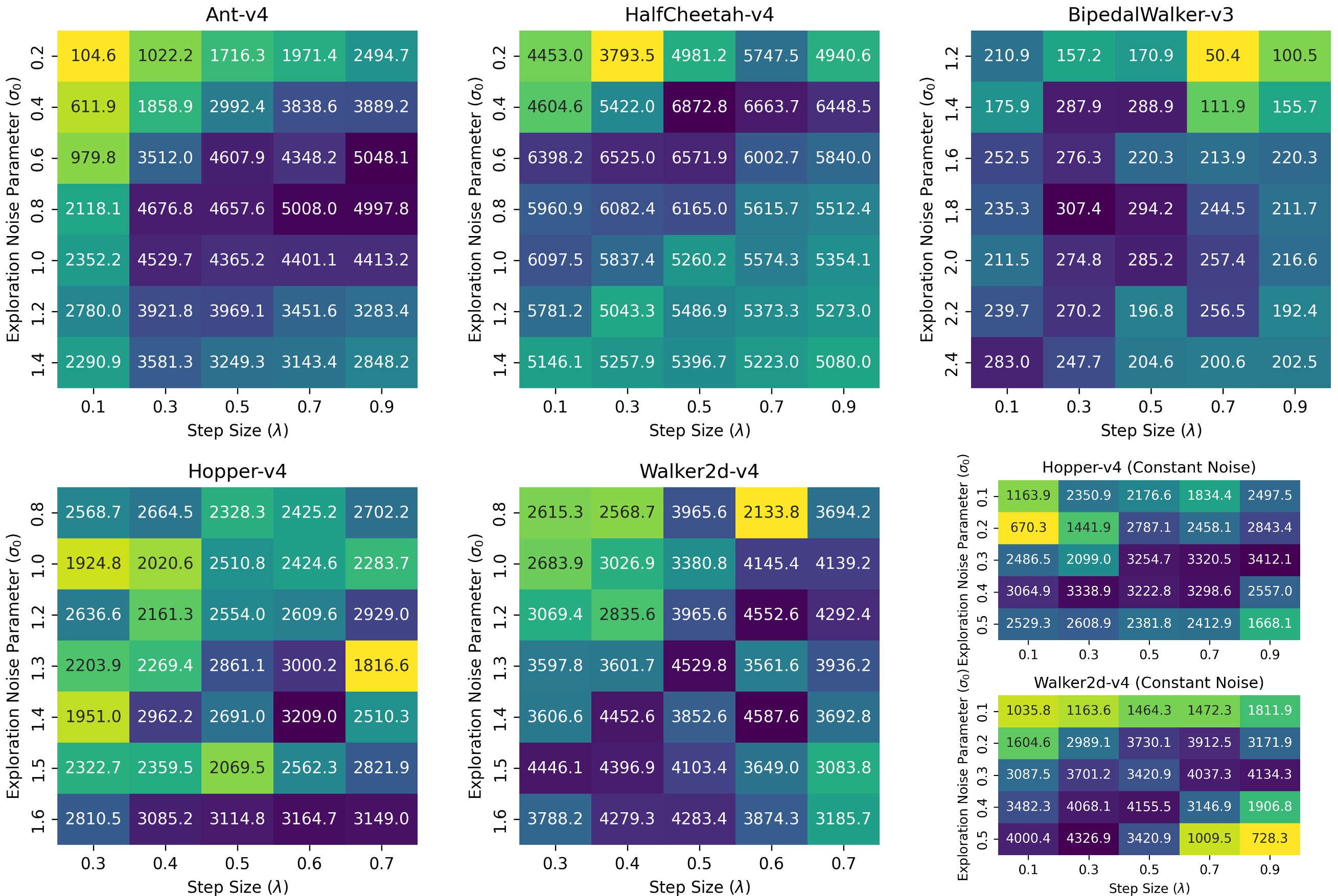}}
    \caption{Sensitivity to exploration noise $\sigma_0$ and step size $\lambda$. Each cell is the final-five-epoch mean testing return. Decreasing-noise panels use five seeds and the constant-noise panels use three, with ten tests per seed. The heatmap is descriptive rather than a confidence-band plot.}
    \label{fig:noise_stepsize}
  \end{center}
\end{figure}

RL performance is often highly sensitive to hyperparameter choices~\citep{hendersonDeepReinforcementLearning2017}, hence fully realizing an algorithm's potential requires careful tuning. Throughout the benchmark, we apply a fixed set of hyperparameters across all environments, with AA-PDA-specific hyperparameters, step size set to $\lambda=0.5$ and the exploration noise standard deviation parameter set to $\sigma_0=1.3$. To better understand the robustness of AA-PDA and the effect of tuning, we conduct a sensitivity study over $\lambda$ and $\sigma_0$ across five continuous-control environments: Hopper, Walker2d, Ant, HalfCheetah, and BipedalWalker (Figure~\ref{fig:noise_stepsize}). Sweeping through a range of parameters shows the sensitivity of AA-PDA with respect to hyperparameter tuning and helps to explain the rationale for parameter selection.

In dynamic balancing tasks (Hopper, Walker2d, BipedalWalker), higher exploration noise is generally favored, which facilitates the discovery of stabilizing behaviors in the state-action space. In contrast, quadrupedal locomotion tasks (Ant and HalfCheetah) are more stable due to multiple contact points to the floor, which require lower exploration noise and larger step sizes to achieve better performance. In this case, increased step sizes provide stronger regularization and improve training efficiency. Interestingly, the initial performance of the Ant-v4 training curve (Figure~\ref{fig:mujoco}) is due to the healthy reward gained from stable behaviors by AA-PDA regularization. The two algorithm-specific hyperparameters ($\lambda$ and $\sigma_0$) have intuitive interpretations about the algorithm's training dynamics. Hence, physical intuitions about task dynamics and stability can guide the identification of effective parameter regimes. Importantly, there exists a broad region of hyperparameters that yields competitive results. Parameter-free methods for RL are an active research topic even in basic settings~\citep{ju2025auto}. Integrating them with AA-PDA is left for future research.

In addition, we explore the option of using constant noise $\sigma(k)=\sigma_0$, illustrated as the last two plots in Figure~\ref{fig:noise_stepsize}. The results indicate that constant noise can also be advantageous, particularly for environments that require stronger and sustained exploration like Hopper and Walker2d. However, it is challenging to identify a single $\sigma_0$ that performs well across all environments at the same time. A decreasing noise schedule traverses a wider range of noise levels, hence making it more adaptive and robust across different tasks.

\subsection{Choice of Optimizer}
\label{sec:app:optimizer}
\label{sec:app:optimzer}

We utilize the SOAP~\citep{vyasSOAPImprovingStabilizing2025} optimizer for neural network training in AA-PDA. SOAP employs Kronecker-factored preconditioning, which calculates and maintains preconditioner matrices that approximate second-order curvature information. As a result, it benefits from a larger batch size and fewer iterations for training. We choose this optimizer for its strong empirical performance in neural network training and because it is readily applicable in RL training environments, even though the theoretical underpinnings are generally lacking.

\begin{figure}[ht]
  \centering
  \includegraphics[width=0.82\columnwidth]{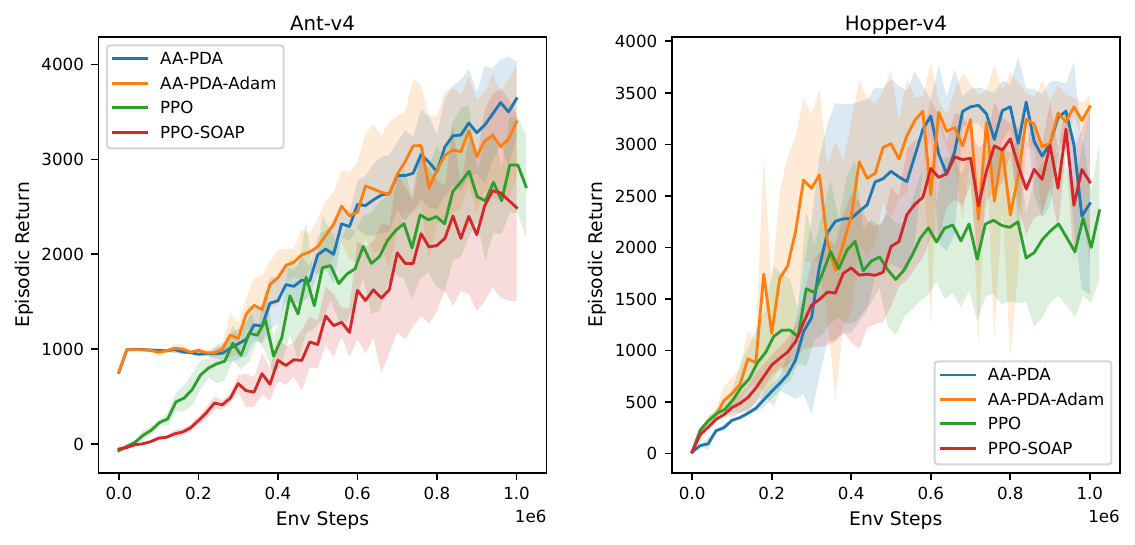}
  \caption{Optimizer ablation on Ant-v4 and Hopper-v4 over five seeds. Curves show the seed-level mean evaluation return with shaded \(\pm 2\,\mathrm{SEM}\) regions. AA-PDA uses SOAP by default, AA-PDA-Adam substitutes Adam for SOAP, and PPO-SOAP substitutes SOAP for PPO's default Adam optimizer.}
  \label{fig:optimizer}
\end{figure}

Figure~\ref{fig:optimizer} compares AA-PDA (using SOAP by default) with AA-PDA-Adam, as well as PPO (using Adam~\citep{kingmaAdamMethodStochastic2017} by default) with PPO-SOAP, for the Ant and Hopper environments. In both benchmarks, AA-PDA and AA-PDA-Adam show similar learning curves, and both variants can achieve higher episodic rewards than PPO. Although SOAP accelerates training in our implementation, it does not provide a consistent improvement in sample efficiency or final performance for AA-PDA relative to Adam. Moreover, substituting Adam with SOAP in PPO does not lead to consistent performance gains. These results indicate that while SOAP is an effective tool for accelerating computation, it is not a primary contributor to AA-PDA's performance advantages.

\section{Actor Approximation and Diagnostics}
\label{sec:app:analysis}

\subsection{Visualization of Optimum Tracking}

We use the inverted pendulum swing-up (Pendulum-v1) environment from Gymnasium~\citep{towersGymnasiumStandardInterface2025} as an example to visualize the behavior of the actor. In this problem, the state is defined as $s = [\cos(\theta), \sin(\theta), \dot{\theta}]$, where $\theta$ and $\dot{\theta}$ represent the pendulum angle and angular velocity, respectively. The goal of this problem is to apply a torque $\tau$ to maintain the pendulum at its upright position. By fixing the angular velocity $\dot{\theta}$, we can visualize the landscape of $\tilde\Psi'$ over the $(\theta, \tau)$ domain, which allows for observing how the actor learns to track the solution of the optimization subproblem. As shown in Figure~\ref{fig:converge}, the first three plots track the evolution of the PDA objective ($\tilde\Psi'$) and the learning trajectory of the actor, confirming that the actor successfully tracks the optimum and solves the problem. The final plot demonstrates that the approximation error of the actor stabilizes during the extended training process.

\begin{figure}[ht]
  \begin{center}
    \centerline{\includegraphics[width=0.8\columnwidth]{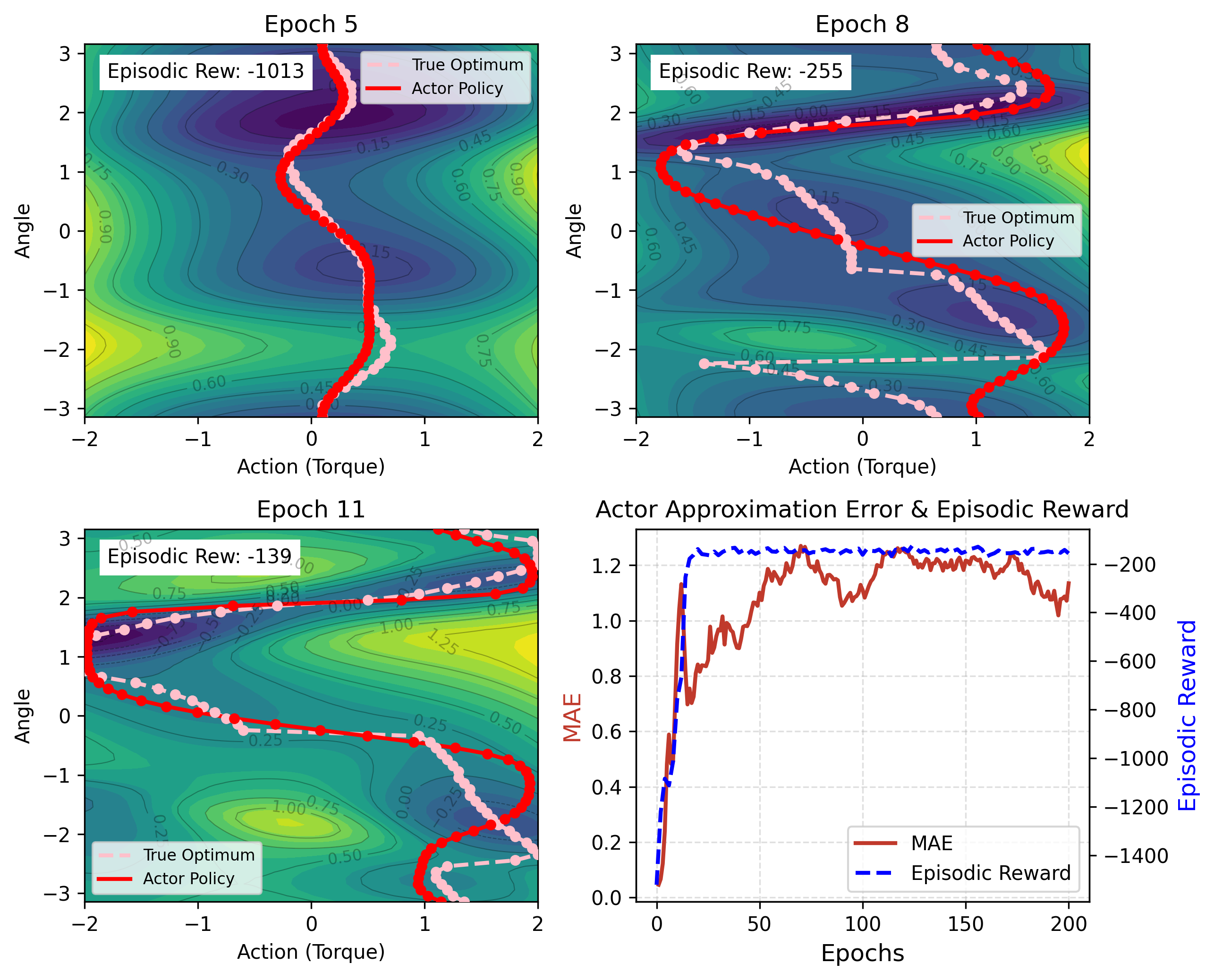}}
    \caption{Visualization of optimum tracking with the actor in the Pendulum-v1 environment. The landscape evolution of the scaled optimization subproblem $\tilde\Psi'$ over epochs 5, 8, and 11, at a fixed $\dot{\theta}=0.2$ is shown in the first three plots. The pink dashed line represents the true optimum of $\tilde\Psi'$, and the solid red line represents the output of the actor network. The final plot tracks the mean absolute error (MAE) between the true optimum and the actor output averaged across a range of states at each epoch for an extended training process.}
    \label{fig:converge}
  \end{center}
\end{figure}

\subsection{Actor Approximation Quality}
\label{sec:app:actorquality}

Similar to the Pendulum-v1 diagnostic, we empirically evaluate the quality of the actor approximation in higher-dimensional environments. For Ant-v4, HalfCheetah-v4, Hopper-v4, and Walker2d-v4, we take snapshots throughout training and compare the actor output with the best-found solution at a set of fixed states. We evaluate 256 states  at 200 checkpoints across 4M environment steps. At each state, we compare the actor output with a best-found reference solution of $\tilde\Psi_k(s,a)$ over $a$, obtained using multistart projected Adam.

Figure~\ref{fig:actor_quality} reports the resulting normalized objective difference. Curves show the mean and median of the gap over the evaluated states, averaged over ten seeds, with $\pm 2\,\mathrm{SEM}$ across seeds. Over this extended training, the normalized gap rises early in training and then levels off rather than diverging. These finite sampled results provide empirical support for the actor approximation quality.

\begin{figure}[ht]
  \centering
  \includegraphics[width=0.82\columnwidth]{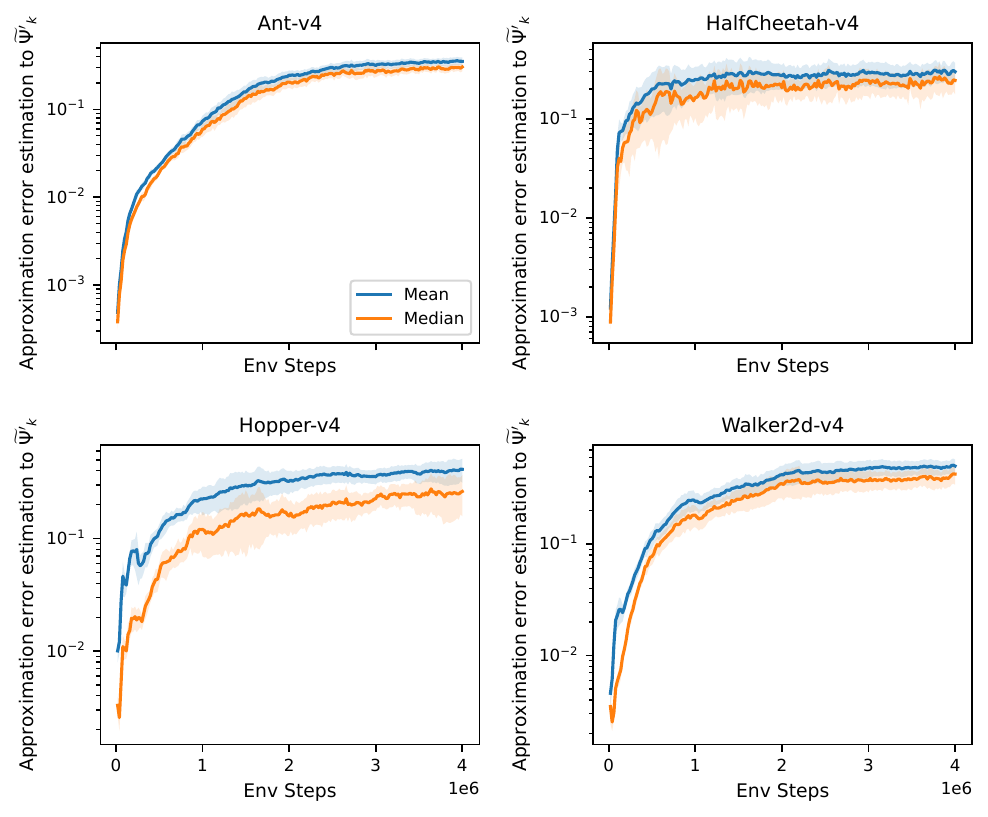}
  \caption{Actor approximation quality on Ant-v4, HalfCheetah-v4, Hopper-v4, and Walker2d-v4 over 4M environment steps. At each of 200 training checkpoints, 256 fixed states are evaluated for each of ten seeds. For each state, the reference action is obtained by selecting the best sampled action and refining it with multistart projected Adam. Curves report the mean and median of the normalized objective gap between the actor action and the reference action over the 256 states, averaged over seeds, with shaded regions showing $\pm 2\,\mathrm{SEM}$ across seeds.}
  \label{fig:actor_quality}
\end{figure}

\subsection{Direct PDA Comparisons with AA-PDA}
\label{sec:app:directpda}
We compare actor-accelerated PDA (AA-PDA) with direct PDA baselines. Within the same framework, PDA-RSGF uses the randomized stochastic gradient-free (RSGF) method~\citep{ghadimi2013stochastic} and PDA-ACFGM uses the auto-conditioned fast gradient method~\citep{li2025simple} to optimize $\tilde\Psi_k$ directly whenever an action is required. Figures~\ref{fig:direct_pda_by_env}--\ref{fig:direct_pda_runtime} compare the performance and training time of the three methods across four MuJoCo environments.

\begin{figure}[H]
  \centering
  \includegraphics[width=0.92\columnwidth]{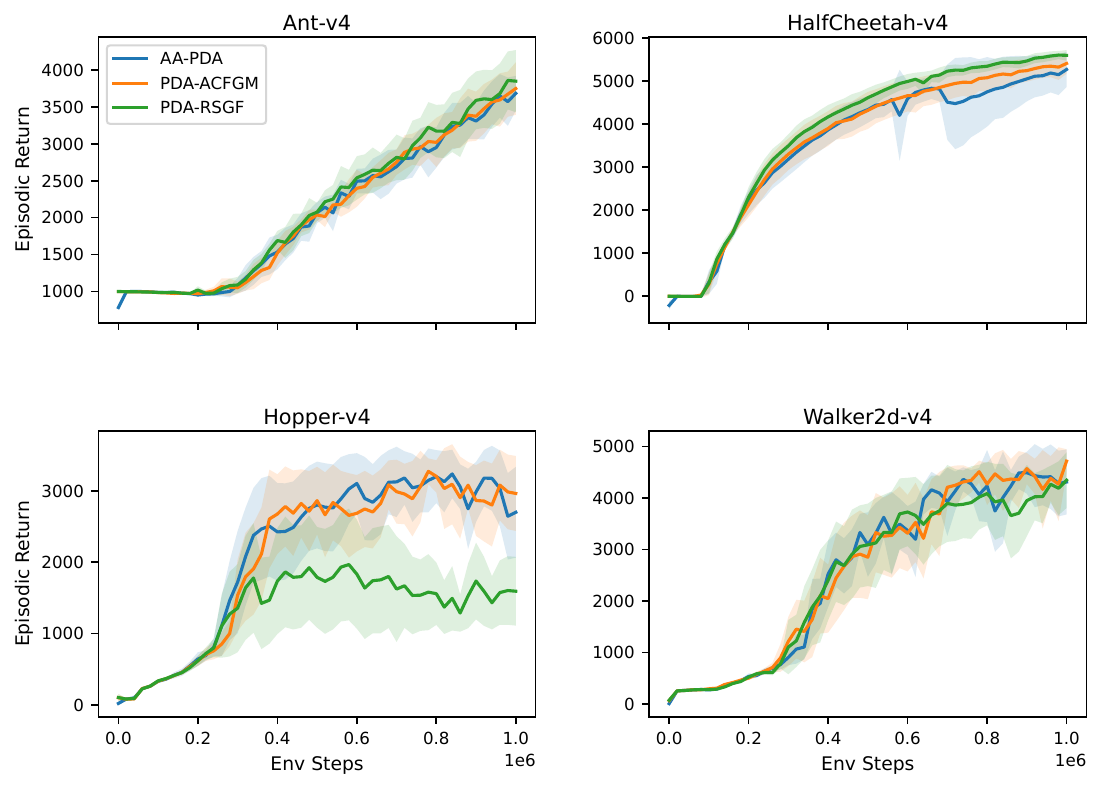}
  \caption{AA-PDA, PDA-RSGF, and PDA-ACFGM returns on Ant-v4, HalfCheetah-v4, Hopper-v4, and Walker2d-v4. Curves show the mean over ten seeds, where each seed is evaluated over ten test episodes. Shaded regions indicate $\pm 2\,\mathrm{SEM}$ across seed-level means.}
  \label{fig:direct_pda_by_env}
\end{figure}

\begin{figure}[H]
  \centering
  \includegraphics[width=0.55\columnwidth]{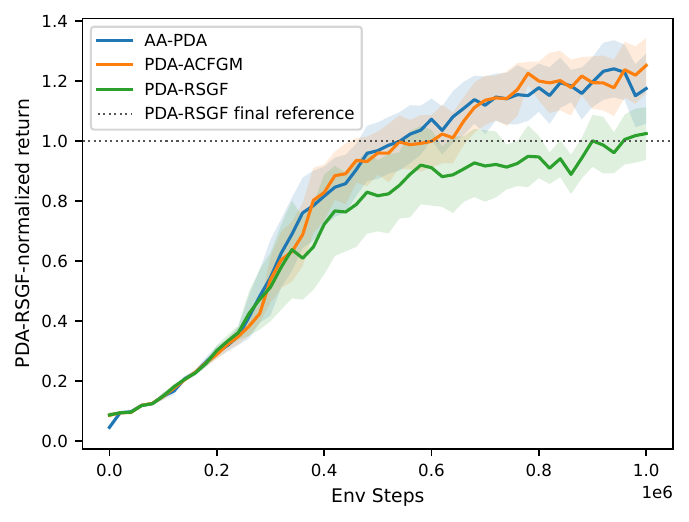}
  \caption{Aggregate return for AA-PDA, PDA-RSGF, and PDA-ACFGM on four MuJoCo-v4 tasks and ten seeds, normalized by the PDA-RSGF final-return reference. Curves show the seed mean with $\pm2\,\mathrm{SEM}$.}
  \label{fig:direct_pda_aggregate}
\end{figure}

\begin{figure}[H]
  \centering
  \includegraphics[width=0.92\columnwidth]{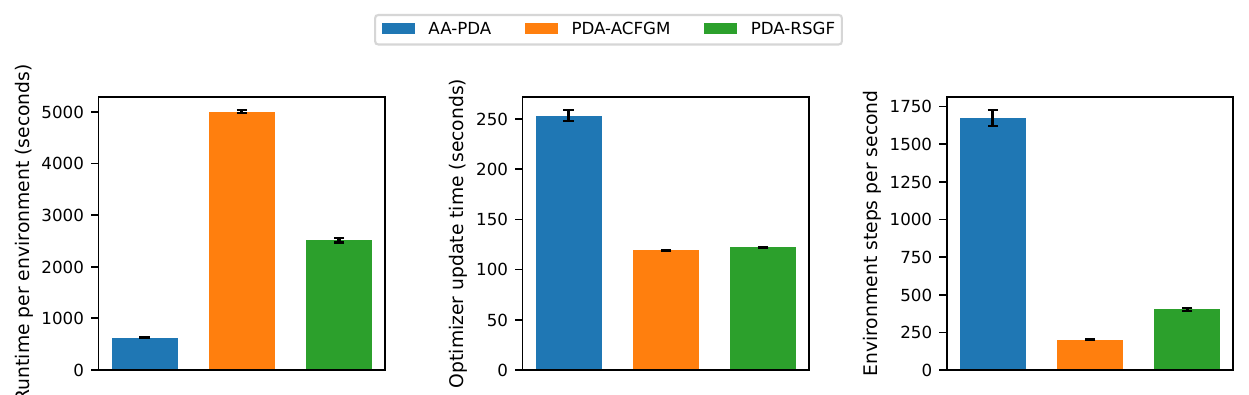}
  \caption{PDA runtime, optimizer-update time, and environment throughput. Values are first averaged over four environments within seed and then summarized across ten matched seeds as mean $\pm2\,\mathrm{SEM}$.}
  \label{fig:direct_pda_runtime}
\end{figure}

At 1M steps, aggregate normalized return is $1.176\pm0.117$ for AA-PDA, $1.026\pm0.088$ for PDA-RSGF, and $1.253\pm0.092$ for PDA-ACFGM. Average runtime per environment is $626.4\pm12.2$ seconds for AA-PDA, compared with $2515.9\pm47.3$ and $5004.0\pm31.4$ seconds for PDA-RSGF and PDA-ACFGM, respectively. Thus, in these experiments, the learned actor provides an approximately $4$--$8$ times speedup with comparable aggregate performance. Moreover, directly optimizing each action with RSGF or ACFGM introduces additional optimizer-specific hyperparameters that must be tuned.

\subsection{Training and Approximation Diagnostics}

\label{sec:app:diagnostics}
\begin{figure}[H]
  \centering
  \includegraphics[width=0.92\columnwidth]{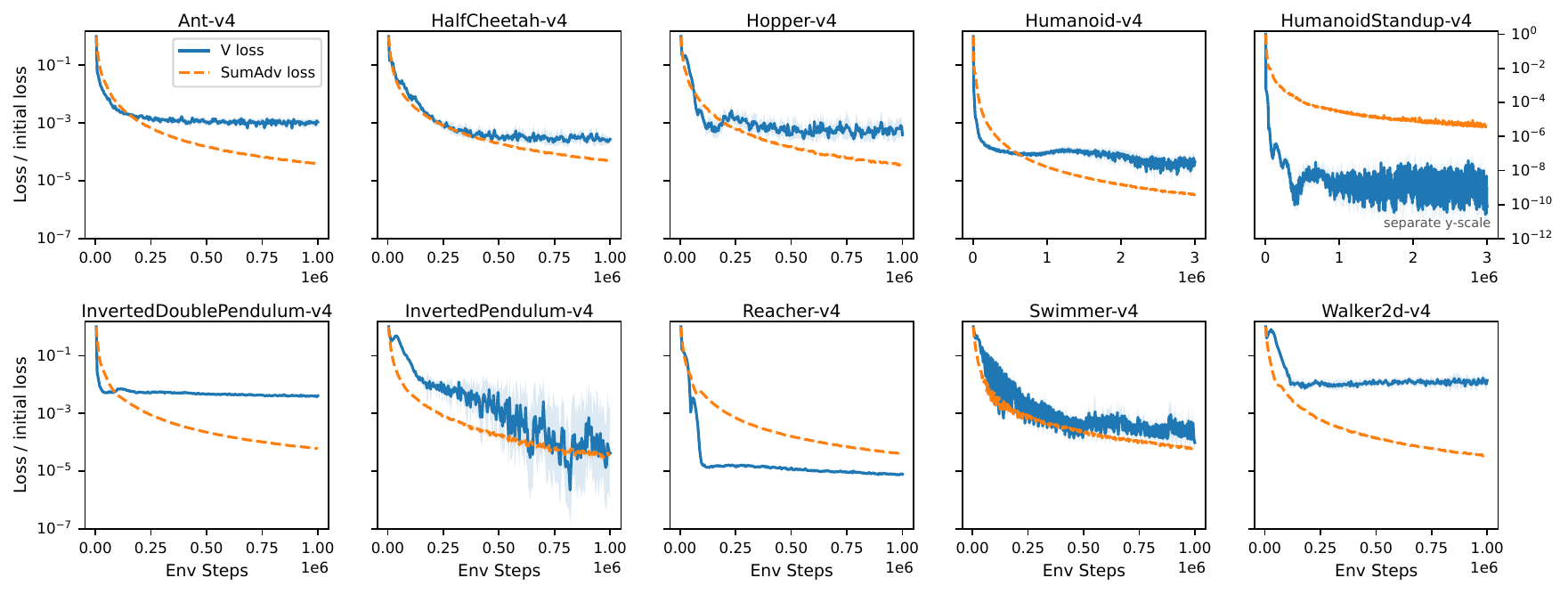}
  \caption{Raw value network and sum advantage network fitting losses on ten MuJoCo-v4 environments and ten seeds. Curves show the seed-level mean, with shaded region indicating $\pm2\,\mathrm{SEM}$.}
  \label{fig:training_fit}
\end{figure}

From an implementation perspective, AA-PDA uses generalized advantage estimation (GAE), as PPO does, so the accuracy of the estimated advantages depends largely on the accuracy of the value-function estimate. As empirical fitting diagnostics, we evaluate the training losses for both the value function and the sum-advantage estimates across MuJoCo environments using ten random seeds. Figure~\ref{fig:training_fit} records the corresponding training curves, and both losses decrease by orders of magnitude across the environments.

\begin{figure}[H]
  \centering
  \includegraphics[width=0.95\columnwidth]{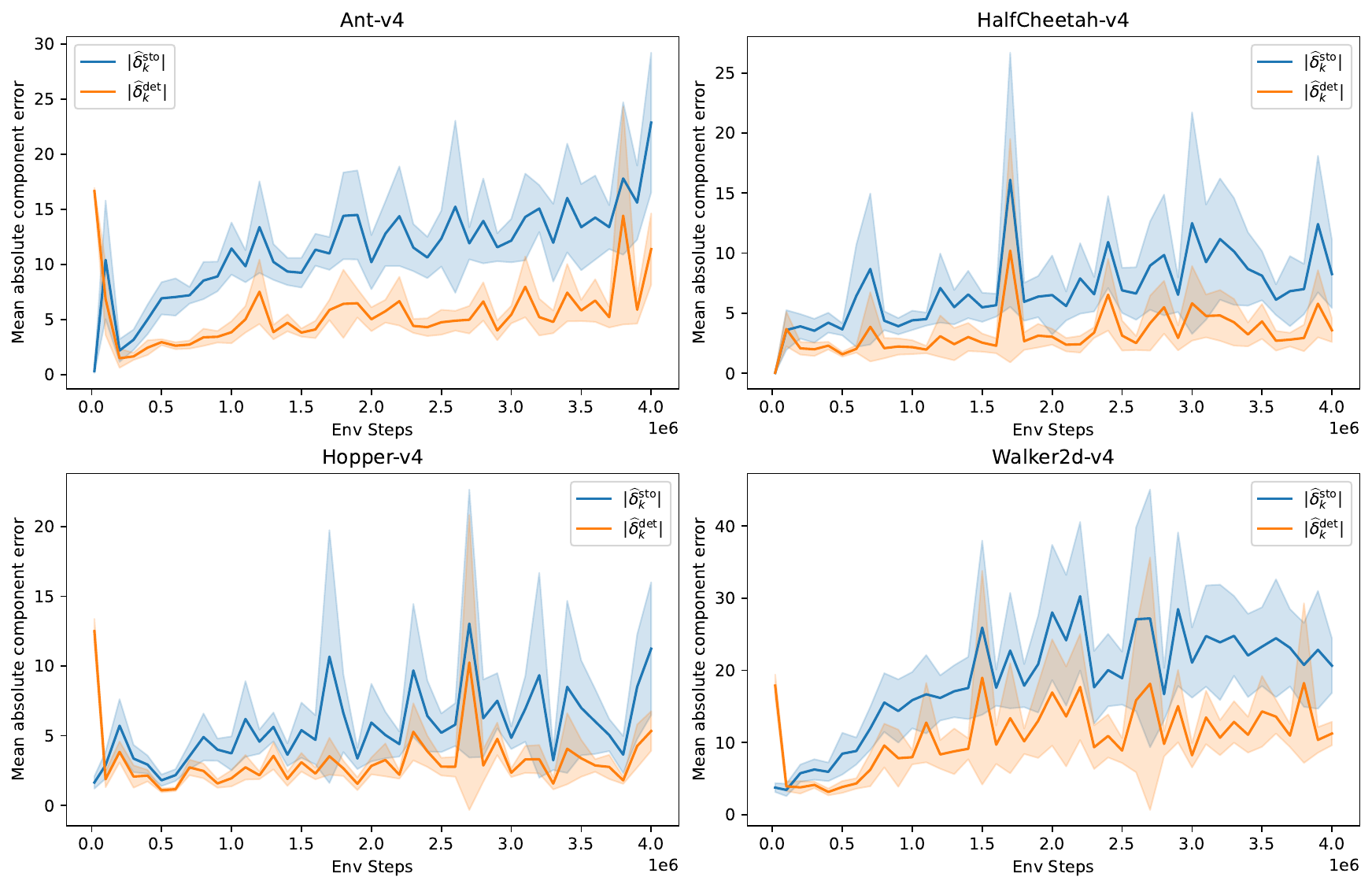}
  \caption{Sampled decomposition of advantage approximation error into stochastic and deterministic components. Curves report the mean over ten seeds with shaded regions showing $\pm2\,\mathrm{SEM}$.}
  \label{fig:delta_components}
\end{figure}

Because the true advantage error cannot be estimated exactly except in simple cases such as finite-state settings, we use sampled states and trajectories to construct rough empirical diagnostics. Figures~\ref{fig:delta_components} and~\ref{fig:varsigma_proxy} evaluate training checkpoints every 20k environment steps up to 4M for Ant-v4, HalfCheetah-v4, Hopper-v4, and Walker2d-v4. At each checkpoint, 16 states are evaluated for each of ten seeds using a rollout horizon of 1,000 steps.

For each seed-checkpoint pair, we approximate the advantage-estimation error as the difference between the GAE estimate from AA-PDA and a 1,000-environment-step rollout advantage. Figure~\ref{fig:delta_components} decomposes this sampled estimate into stochastic and deterministic components following their definitions. Note that no separate bound is imposed on  $\delta_k^{sto}$, and $\delta_k^{det}$ is bounded only at the actor (executed-policy) and optimal actions through $\varsigma$ in Assumption~\ref{assumption:1}. Across the four environments, the sampled stochastic component is generally larger and more variable than the deterministic component, although their magnitudes and evolution are task-dependent.

\begin{figure}[H]
  \centering
  \includegraphics[width=0.95\columnwidth]{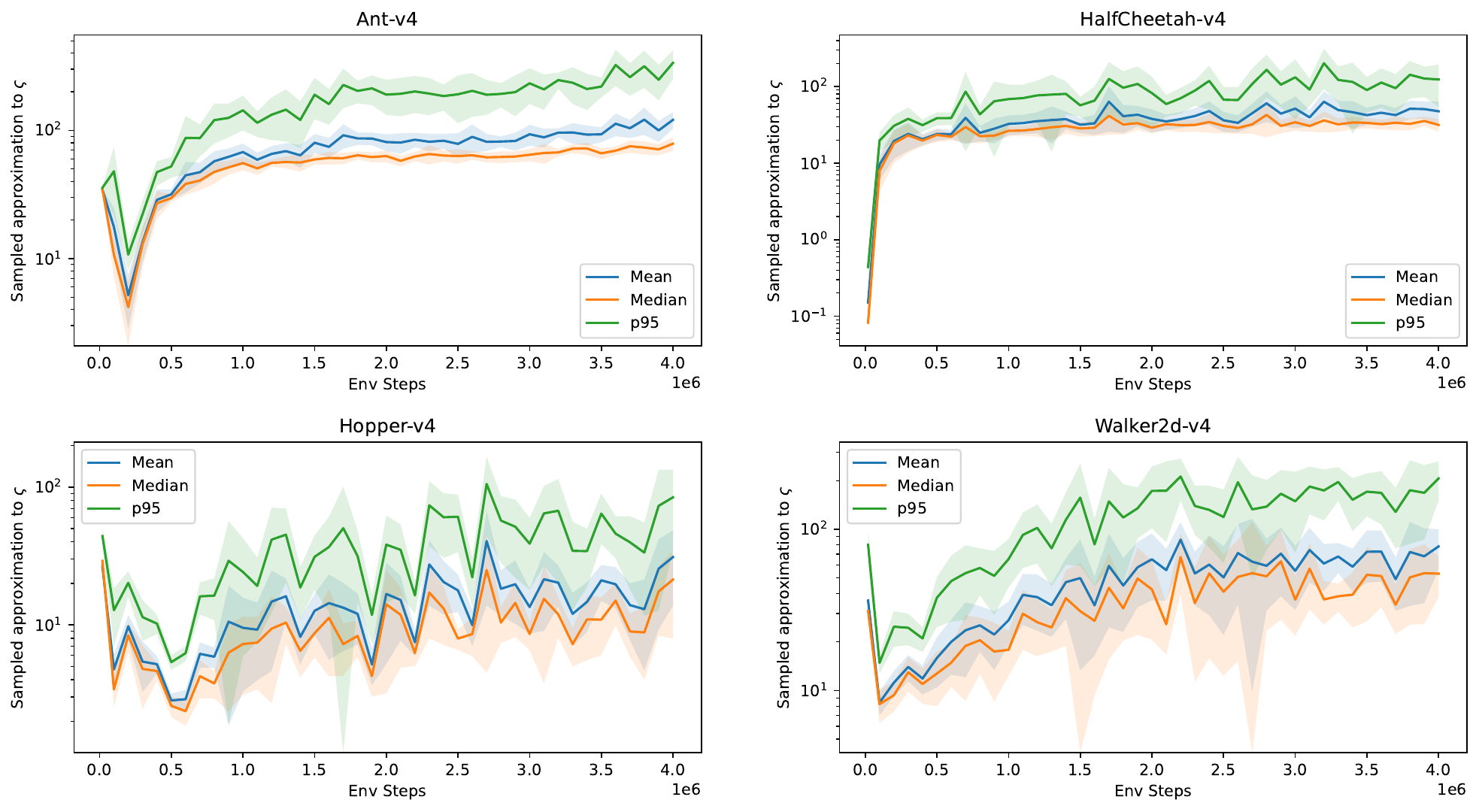}
  \caption{Sampled approximation to $\varsigma$. For each state, the finite candidate set contains the best-found reference action and eight alternatives. Mean, median, and p95 are computed over 16 states within each seed. Curves then report the across-seed mean with shaded regions showing $\pm2\,\mathrm{SEM}$. The p95 is the point statistic, while the shading is its seed-level uncertainty.}
  \label{fig:varsigma_proxy}
\end{figure}

The constant $\varsigma$  is the sum of two deterministic error sources: function-approximation error and bias of the learned advantage function. In general, these terms cannot be evaluated exactly because their values depend on the unknown optimal solution. Moreover, when a neural network is used for policy evaluation, computing the corresponding globally optimal approximation is generally non-convex and may be computationally intractable. The analysis therefore assumes that these terms are absolutely bounded through Assumption~\ref{assumption:1}. Similar to the advantage-error diagnostic, we use sampled states and trajectories to construct an empirical diagnostic for $\varsigma$, following the definition in Assumption~\ref{assumption:1} with $\pi^*$ approximated by the best-found action. In Figure~\ref{fig:varsigma_proxy}, the 95th-percentile curve over sampled states remains within a finite observed range over 4M steps, substantially beyond the 1M-step benchmarking horizon, providing empirical support for the assumption.

\newpage
\section{Benchmark Results}
\label{sec:app:experiments}

\subsection{Algorithm Runtime on MuJoCo-v4}
\label{sec:app:runtime}

\begin{table}[H]
\centering
\caption{Wall-clock comparisons for the algorithms. The upper panel reports five parallel fixed-budget runs on an Intel i7-14700K. SAC and TD3 additionally use an NVIDIA RTX 4070. Times are seconds, reported as mean $\pm$ standard deviation, and ratios are relative to AA-PDA within each environment group. The lower panel reports the matched direct-PDA comparison over Ant, HalfCheetah, Hopper, and Walker2d: runtime is first averaged over the four environments within seed and then summarized over ten seeds as mean $\pm2\,\mathrm{SEM}$.}
\label{tab:runtime}
\footnotesize
\begin{tabular}{lcc}
\toprule
\textbf{Algorithm} & \textbf{MuJoCo-v4 (w/o humanoid)} & \textbf{MuJoCo-v4 (humanoid variants)} \\
\midrule
AA-PDA & 423.4 $\pm$ 104.8 \textbf{(1.00x)} & 1480.7 $\pm$ 178.2 \textbf{(1.00x)} \\
PPO  & 932.1 $\pm$ 120.1 \textbf{(2.20x)} & 1960.8 $\pm$ 167.2 \textbf{(1.32x)} \\
NPG  & 344.1 $\pm$ 108.4 \textbf{(0.81x)} & 1011.3 $\pm$ 137.0 \textbf{(0.68x)} \\
TRPO & 349.5 $\pm$ 112.6 \textbf{(0.83x)} & 1009.6 $\pm$ 127.6 \textbf{(0.68x)} \\
SAC  & 20335.9 $\pm$ 1878.3 \textbf{(48.03x)} & 19531.4 $\pm$ 2181.6 \textbf{(13.19x)} \\
TD3  & 11571.7 $\pm$ 510.6 \textbf{(27.33x)} & 29113.7 $\pm$ 596.8 \textbf{(19.66x)} \\
\midrule
\multicolumn{3}{l}{\textit{Matched direct-PDA comparison (four environments, ten seeds)}}\\
\textbf{Algorithm} & \textbf{Runtime/environment (seconds)} & \textbf{Ratio to AA-PDA}\\
\midrule
AA-PDA & 626.4 $\pm$ 12.2 & \textbf{1.00x}\\
PDA-RSGF & 2515.9 $\pm$ 47.3 & \textbf{4.02x}\\
PDA-ACFGM & 5004.0 $\pm$ 31.4 & \textbf{7.99x}\\
\bottomrule
\end{tabular}
\end{table}

\subsection{On-Policy Benchmark Results}
\begin{table}[H]
\centering
\caption{Performance comparison with on-policy baselines. Each cell shows the mean $\pm2\,\mathrm{SEM}$ over ten seed-level means, after averaging the final five evaluations per seed. All environments are trained for 1M environment steps unless marked (3M).}
\label{tab:benchmark}
\scriptsize 
\setlength{\tabcolsep}{3.5pt}
\begin{tabular}{lcccc}
\toprule
Environment & AA-PDA& NPG & PPO & TRPO \\
\midrule
HalfCheetah-v4 & \textbf{5174.6 $\pm$ 456.2} & 3556.4 $\pm$ 553.1 & 4067.5 $\pm$ 346.1 & 4496.8 $\pm$ 636.7 \\
Ant-v4 & \textbf{3568.2 $\pm$ 285.6} & 2002.5 $\pm$ 131.0 & 2589.5 $\pm$ 244.0 & 2807.1 $\pm$ 387.4 \\
Hopper-v4 & \textbf{2944.3 $\pm$ 454.7} & 1650.8 $\pm$ 460.9 & 2329.7 $\pm$ 354.3 & 2017.0 $\pm$ 548.1 \\
Walker2d-v4 & \textbf{4367.1 $\pm$ 547.5} & 2923.2 $\pm$ 435.7 & 3277.0 $\pm$ 457.6 & 4128.8 $\pm$ 351.3 \\
InvertedPendulum-v4 & \textbf{1000.0 $\pm$ 0.0} & \textbf{1000.0 $\pm$ 0.0} & \textbf{993.3} $\pm$ 13.5 & 380.3 $\pm$ 257.7 \\
InvertedDoublePendulum-v4 & \textbf{9167.5 $\pm$ 366.6} & 7967.2 $\pm$ 516.9 & 7926.8 $\pm$ 637.9 & 2723.4 $\pm$ 1469.0 \\
Reacher-v4 & \textbf{-4.0 $\pm$ 0.1} & -5.7 $\pm$ 0.4 & -18.7 $\pm$ 11.7 & -6.1 $\pm$ 0.4 \\
Swimmer-v4 & \textbf{111.3 $\pm$ 19.5} & 25.1 $\pm$ 6.1 & 54.2 $\pm$ 9.3 & 35.0 $\pm$ 18.0 \\
Humanoid-v4 & \textbf{760.1 $\pm$ 91.1} & \textbf{745.1} $\pm$ 87.1 & 669.2 $\pm$ 67.6 & 719.5 $\pm$ 66.2 \\
Humanoid-v4 (3M) & \textbf{5020.0 $\pm$ 277.9} & 4650.6 $\pm$ 386.0 & 933.3 $\pm$ 191.4 & 4745.1 $\pm$ 339.0 \\
HumanoidStandup-v4 & \textbf{161184.5 $\pm$ 1839.7} & 38364.7 $\pm$ 2454.5 & 135853.7 $\pm$ 12646.5 & 36737.7 $\pm$ 1413.0 \\
\midrule
LunarLander-v3 & \textbf{204.7 $\pm$ 25.8} & 34.1 $\pm$ 37.5 & \textbf{204.4} $\pm$ 24.3 & -83.0 $\pm$ 33.2 \\
BipedalWalker-v3 & 149.0 $\pm$ 69.5 & 212.3 $\pm$ 21.5 & \textbf{251.6 $\pm$ 20.7} & 106.3 $\pm$ 67.8 \\
\midrule
NewsvendorEnv-v0 & \textbf{-1.52e5$\pm$9.32e4} & -2.74e7$\pm$3.62e7 & -1.92e5$\pm$2.40e5 & -3.37e7$\pm$3.45e7 \\
NewsvendorEnv-v0 (3M) & \textbf{2.19e4$\pm$0.73e4} & -1.02e7$\pm$1.54e7 & -3.72e4$\pm$2.65e4 & -5.26e7$\pm$6.75e7 \\
PortfolioOptEnv-v0 & \textbf{10550.9 $\pm$ 367.9} & 10095.3 $\pm$ 543.2 & 10062.3 $\pm$ 360.9 & 10195.6 $\pm$ 510.9 \\
InvManagementBacklogEnv-v0 & 491.6 $\pm$ 5.1 & 430.4 $\pm$ 20.4 & \textbf{496.4 $\pm$ 2.6} & 397.5 $\pm$ 16.6 \\
InvManagementLostSalesEnv-v0 & \textbf{472.3 $\pm$ 4.0} & 432.3 $\pm$ 4.0 & 447.4 $\pm$ 10.4 & 431.5 $\pm$ 2.7 \\
\bottomrule
\end{tabular}
\end{table}

\subsection{Training Curves for OR-Gym} \label{sec:app:or}
\begin{figure}[ht]
  \centering
  \includegraphics[width=\columnwidth]{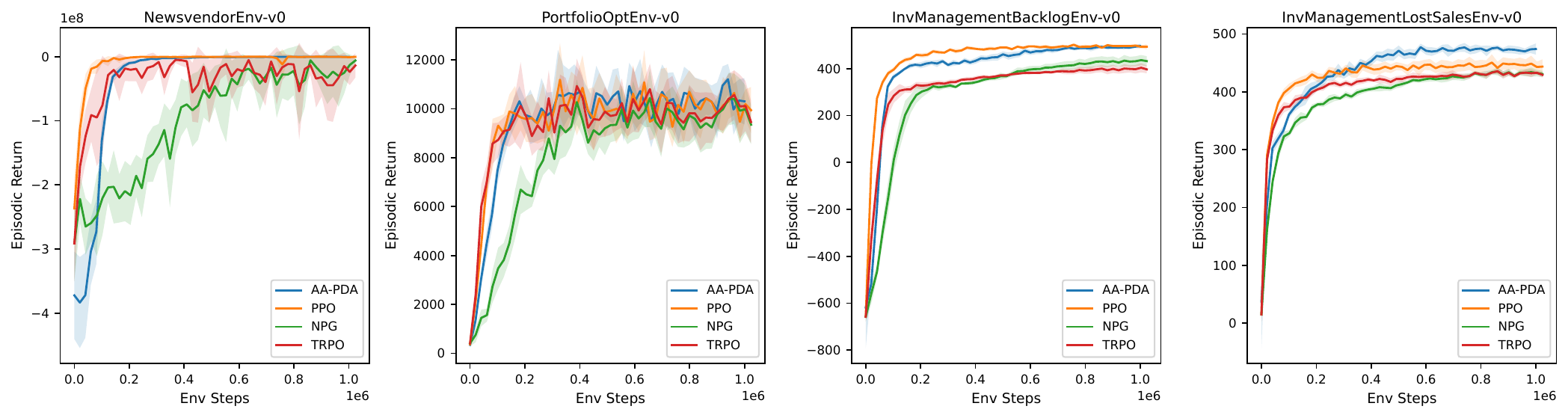}
  \caption{AA-PDA, PPO, NPG, and TRPO learning curves on four OR-Gym environments.  Curves show the mean across ten seeds, where each seed is evaluated over ten test episodes. Shaded regions indicate $\pm 2\,\mathrm{SEM}$ across seed-level means.}
  \label{fig:or_learning}
\end{figure}

\subsection{Benchmark Results and Training Curves for Comparison with Off-Policy Methods}
\label{sec:app:offpolicy}

\begin{table}[H]
\centering
\caption{Performance comparison with off-policy baselines. Each cell shows the mean $\pm2\,\mathrm{SEM}$ over ten seed-level means, after averaging the final five evaluations per seed. All environments are trained for 1M environment steps unless marked (3M).}
\label{tab:benchmark_offpolicy}
\scriptsize 
\begin{tabular}{lccc}
\toprule
Environment & AA-PDA & SAC & TD3 \\
\midrule
HalfCheetah-v4 & 5174.6 $\pm$ 456.2 & \textbf{11892.1 $\pm$ 494.2} & 10228.4 $\pm$ 417.5 \\
Ant-v4 & 3568.2 $\pm$ 285.6 & \textbf{5764.8 $\pm$ 167.0} & 3914.1 $\pm$ 881.7 \\
Hopper-v4 & 2944.3 $\pm$ 454.7 & \textbf{3387.0 $\pm$ 125.6} & 2896.6 $\pm$ 465.8 \\
Walker2d-v4 & \textbf{4367.1 $\pm$ 547.5} & 3864.3 $\pm$ 509.6 & 3435.5 $\pm$ 381.2 \\
InvertedPendulum-v4 & \textbf{1000.0 $\pm$ 0.0} & \textbf{1000.0 $\pm$ 0.0} & 783.5 $\pm$ 261.6 \\
InvertedDoublePendulum-v4 & 9167.5 $\pm$ 366.6 & \textbf{9183.2 $\pm$ 312.9} & 8134.8 $\pm$ 1833.3 \\
Reacher-v4 & -4.0 $\pm$ 0.1 & \textbf{-3.6 $\pm$ 0.1} & -3.7 $\pm$ 0.1 \\
Swimmer-v4 & \textbf{111.3 $\pm$ 19.5} & 43.5 $\pm$ 0.8 & 93.3 $\pm$ 16.3 \\
Humanoid-v4 & 760.1 $\pm$ 91.1 & 4785.3 $\pm$ 372.3 & \textbf{4846.9 $\pm$ 124.9} \\
HumanoidStandup-v4 & \textbf{161184.5 $\pm$ 1839.7} & 146249.0 $\pm$ 8054.1 & 107311.1 $\pm$ 12563.0 \\
\midrule
LunarLander-v3 & 204.7 $\pm$ 25.8 & \textbf{284.6 $\pm$ 1.2} & 229.0 $\pm$ 64.7 \\
BipedalWalker-v3 & 149.0 $\pm$ 69.5 & -9.1 $\pm$ 5.6 & \textbf{295.2 $\pm$ 8.4} \\
\midrule
NewsvendorEnv-v0 & -152387.9 $\pm$ 93157.7 & \textbf{-19212.1 $\pm$ 1583.3} & -19758.2 $\pm$ 844.6 \\
PortfolioOptEnv-v0 & \textbf{10550.9 $\pm$ 367.9} & 9010.2 $\pm$ 476.1 & --- \\
InvManagementBacklogEnv-v0 & \textbf{491.6 $\pm$ 5.1} & 281.8 $\pm$ 24.4 & -1575.7 $\pm$ 443.5 \\
InvManagementLostSalesEnv-v0 & \textbf{472.3 $\pm$ 4.0} & 356.5 $\pm$ 6.1 & -311.9 $\pm$ 312.9 \\
\bottomrule
\end{tabular}
\end{table}

\begin{figure*}[ht]
  \centering
  \includegraphics[width=\textwidth]{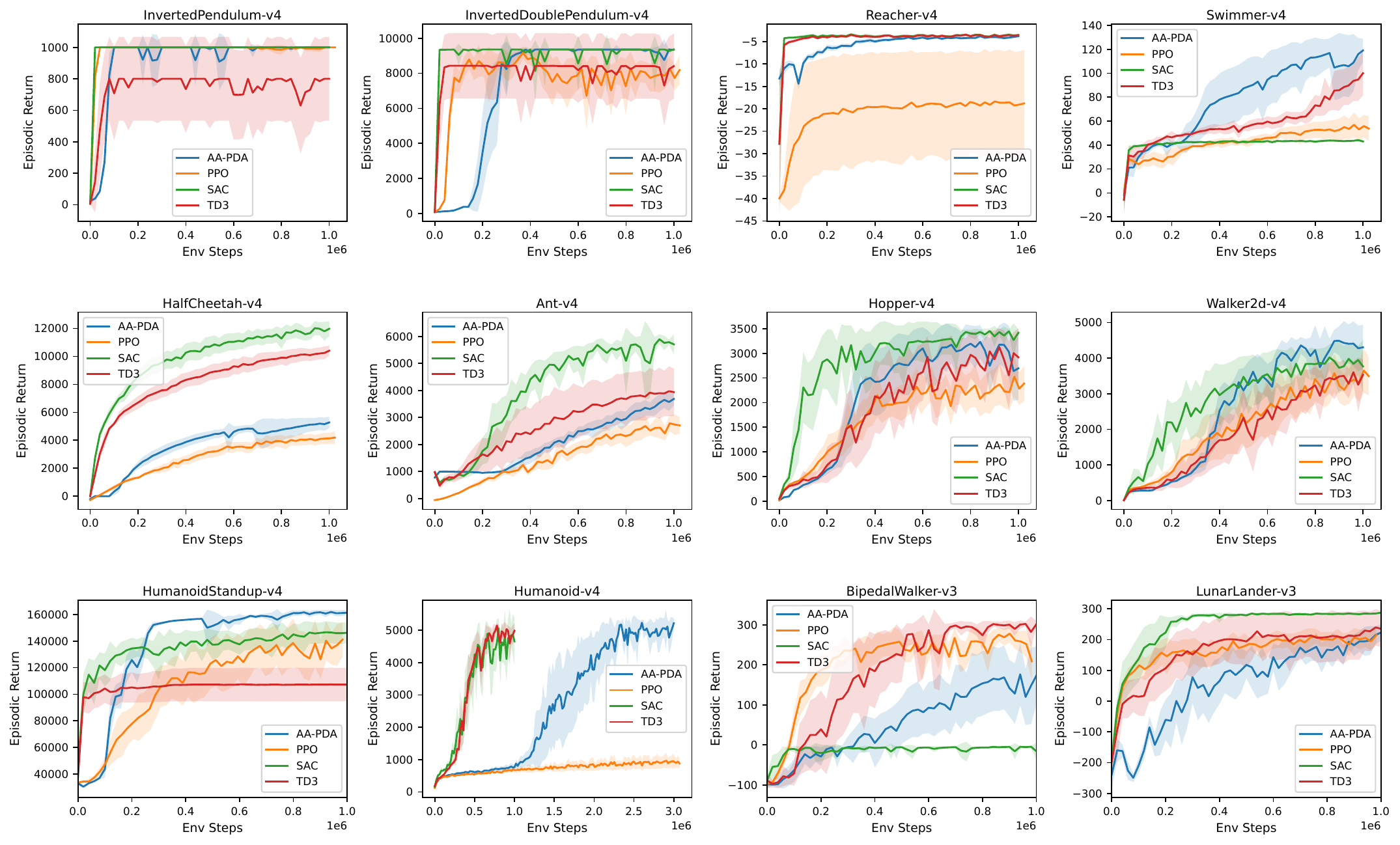}
  \caption{AA-PDA, PPO, SAC, and TD3 learning curves on MuJoCo-v4 and Box2D. Curves show the mean over ten seeds, where each seed is evaluated over ten test episodes. Shaded regions indicate $\pm 2\,\mathrm{SEM}$ across seed-level means.}
  \label{fig:mujoco_offpolicy}
\end{figure*}

\begin{figure}[ht]
  \centering
  \includegraphics[width=\columnwidth]{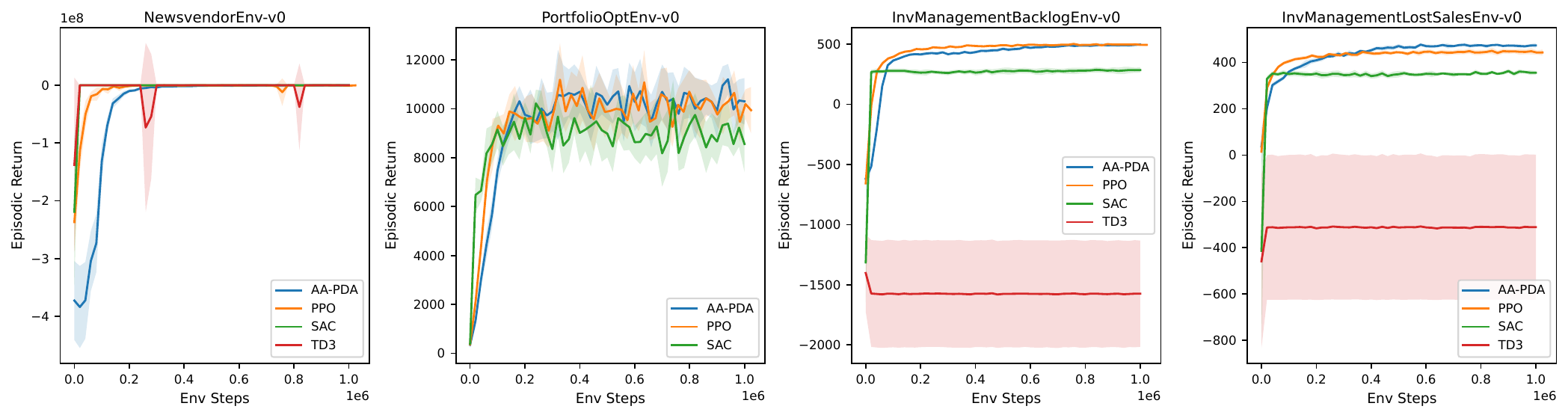}
  \caption{AA-PDA, PPO, SAC, and TD3 learning curves on OR-Gym. Curves show the mean over ten seeds, where each seed is evaluated over ten test episodes. Shaded regions indicate $\pm 2\,\mathrm{SEM}$ across seed-level means.}
  \label{fig:or_offpolicy}
\end{figure}

Table~\ref{tab:benchmark_offpolicy} reports the final returns. Figures~\ref{fig:mujoco_offpolicy} and~\ref{fig:or_offpolicy} show that no method dominates every task. SAC has the strongest aggregate continuous-control result, while AA-PDA and TD3 have overlapping aggregate IQM intervals and differ substantially across individual environments. On OR-Gym, SAC is strongest on Newsvendor, AA-PDA is strongest on PortfolioOpt and both InvManagement settings.

\newpage
\subsection{Architectural and Off-Policy Advances}

Our experiments use simple MLP architectures and established baselines rather than the latest state-of-the-art agents, keeping the evaluation focused on the contribution of AA-PDA itself. Recent advances such as SimBa and SimBaV2 study scalable architectures and normalization~\citep{lee2025simba,lee2025hyperspherical}, TD7 augments TD3 with state-action representations and checkpointing~\citep{fujimoto2023sale}, and MR.Q develops a general-purpose model-free agent~\citep{fujimoto2025mrq}. Their architectural and off-policy advances may be complementary to AA-PDA. However, these methods are not evaluated in this study.

An off-policy version of AA-PDA would introduce additional challenges in maintaining the theoretical guarantees of AA-PDA. In PDA, each advantage function is defined with respect to the current policy, while the sum-advantage function accumulates information across all past policies. Introducing replay would mix data generated by different behavior policies, raising the question of how to preserve the interpretation of both the sum advantage function and the resulting cumulative objective. Although replay could improve data reuse, it would also alter the assumptions underlying the current convergence analysis and require additional algorithmic modification.

\subsection{Additional Benchmark Results for MuJoCo-v5}
\begin{table}[!ht]
\centering
\caption{Performance comparison with on-policy baselines for MuJoCo-v5. Each cell shows the mean $\pm2\,\mathrm{SEM}$ over ten seed-level means, after averaging the final five evaluations per seed. All environments are trained for 1M environment steps unless marked (3M).}
\label{tab:benchmark_mujoco_v5}
\small 
\begin{tabular}{lcc}
\toprule
Environment & AA-PDA& PPO \\
\midrule
HalfCheetah-v5 & \textbf{5174.6 $\pm$ 456.2}& 4067.5 $\pm$ 346.1 \\
Ant-v5 & \textbf{2625.1 $\pm$ 232.8}& 1115.4 $\pm$ 155.5 \\
Hopper-v5 & \textbf{2693.9 $\pm$ 534.3}& 2397.8 $\pm$ 116.6 \\
Walker2d-v5 & 3642.3 $\pm$ 844.2& \textbf{3796.7 $\pm$ 543.7} \\
InvertedPendulum-v5 & 916.0 $\pm$ 135.4& \textbf{975.8 $\pm$ 48.4} \\
InvertedDoublePendulum-v5 & \textbf{9349.2 $\pm$ 6.3}& 6782.6 $\pm$ 864.2 \\
Reacher-v5 & \textbf{-3.9 $\pm$ 0.1}& -11.1 $\pm$ 8.5 \\
Swimmer-v5 & \textbf{111.3 $\pm$ 19.5}& 54.2 $\pm$ 9.3 \\
Humanoid-v5 & \textbf{793.1 $\pm$ 87.8}& 671.1 $\pm$ 79.1 \\
Humanoid-v5 (3M) & \textbf{5180.1 $\pm$ 214.9}& 941.5 $\pm$ 182.6 \\
HumanoidStandup-v5 & \textbf{161199.0 $\pm$ 1893.1}& 130309.3 $\pm$ 17096.2 \\
\bottomrule
\end{tabular}
\end{table}

\begin{figure}[H]
  \centering
  \includegraphics[width=0.92\columnwidth]{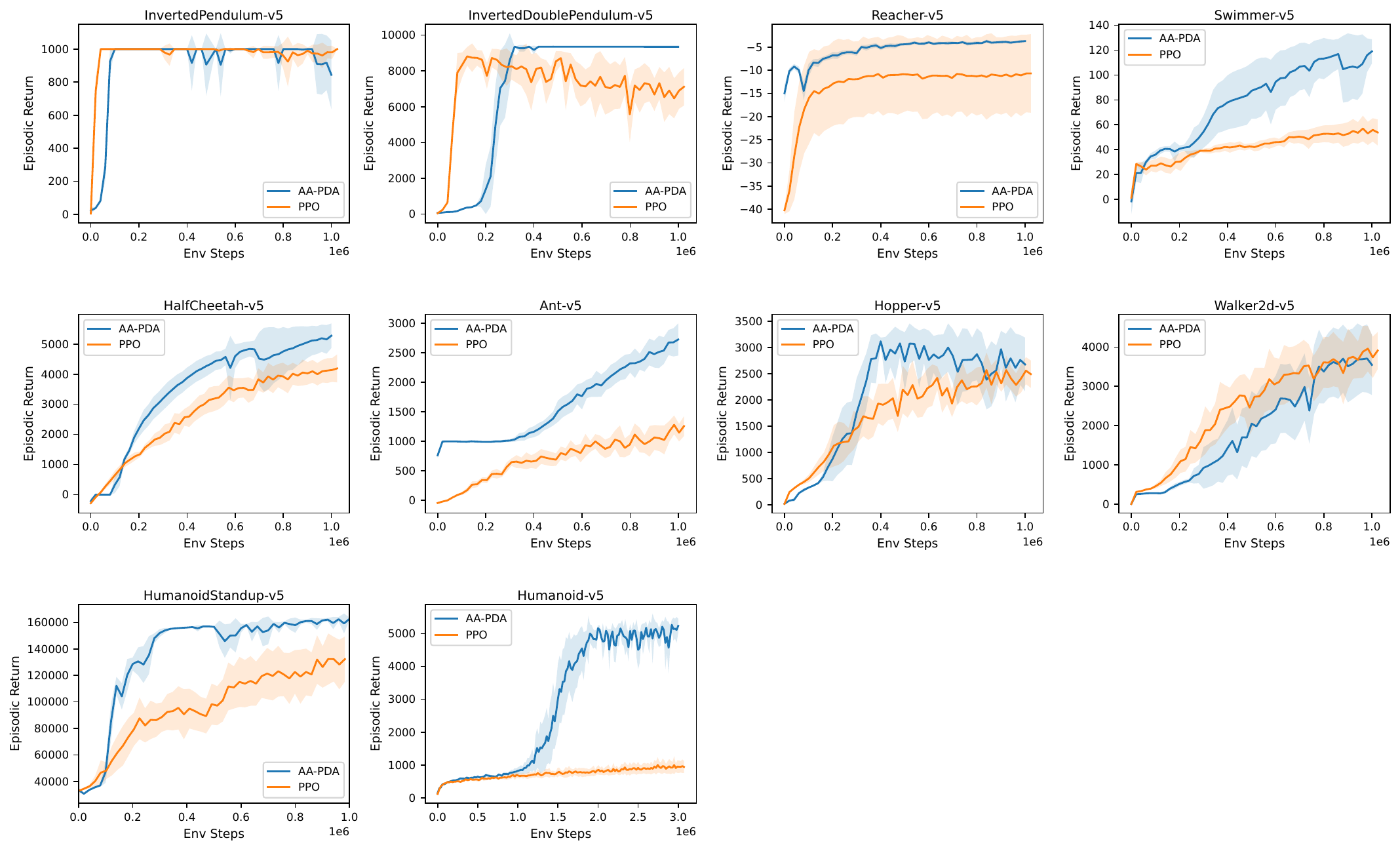}
  \caption{AA-PDA and PPO learning curves on MuJoCo-v5. Curves show the mean over ten seeds, where each seed is evaluated over ten test episodes. Shaded regions indicate $\pm 2\,\mathrm{SEM}$ across seed-level means.}
  \label{fig:mujocov5}
\end{figure}

\subsection{Aggregate Benchmark Summaries}
\label{sec:app:aggregate}
\begin{figure}[H]
  \centering
  \includegraphics[width=0.96\columnwidth]{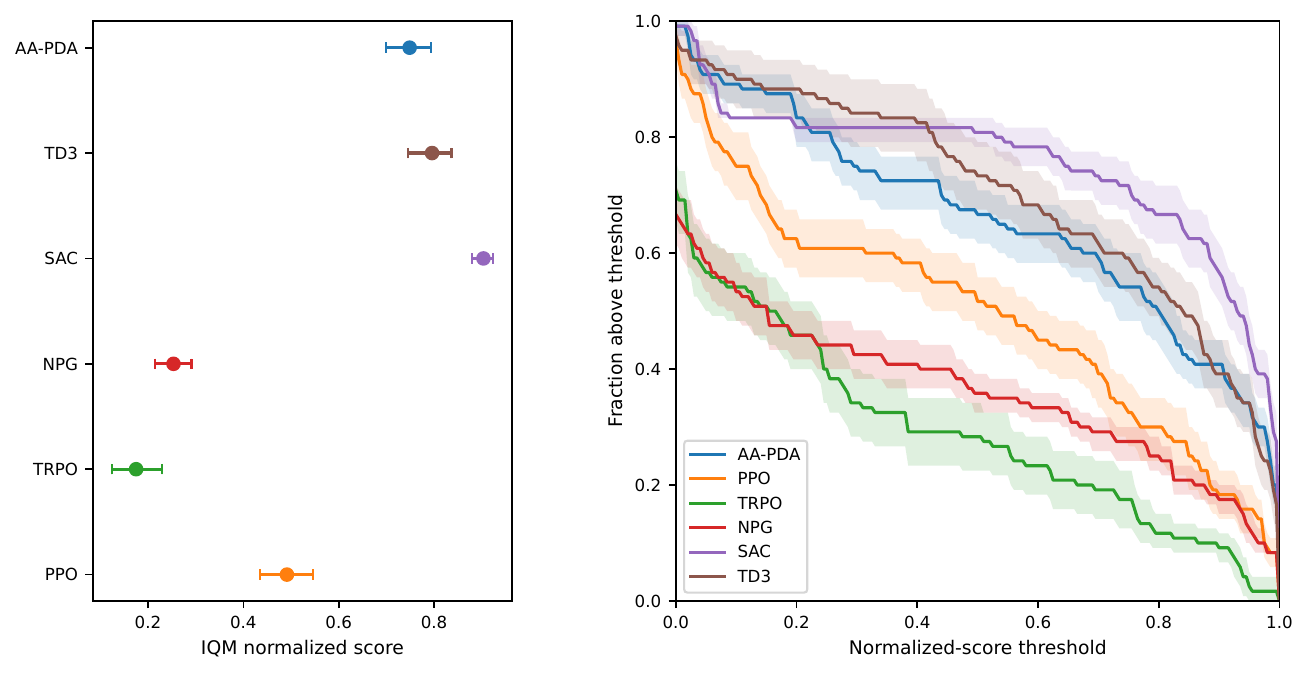}
  \caption{Aggregate final performance over 12 MuJoCo-v4 and Box2D environments with ten seeds per environment. Returns are normalized using frozen per-environment min--max references. Left: interquartile mean (IQM). Right: fraction of runs exceeding each normalized-score threshold. Bands and error bars are 95\% stratified-bootstrap intervals from 5,000 replicates, resampling seeds within environments.}
  \label{fig:aggregate_benchmark}
\end{figure}

\begin{figure}[H]
  \centering
  \includegraphics[width=0.55\columnwidth]{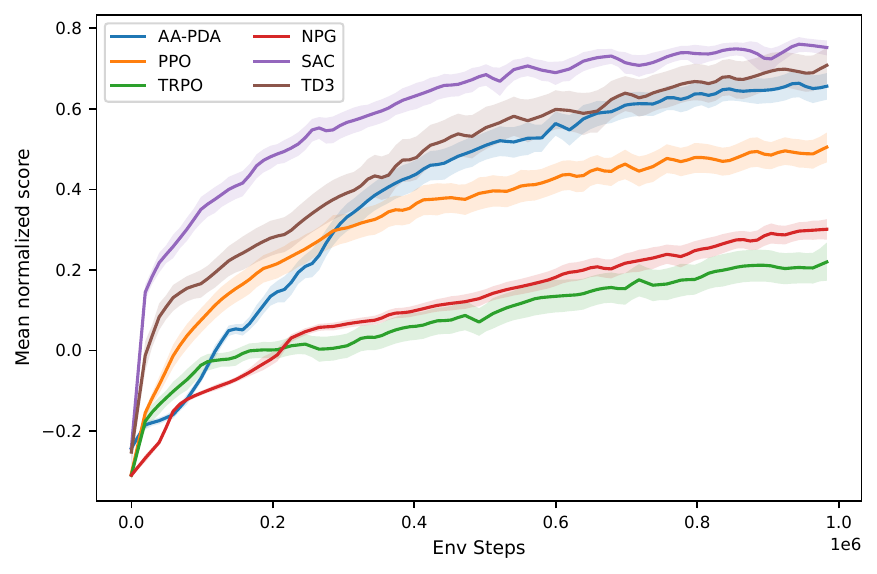}
  \caption{Aggregated normalized learning curve across the 12 continuous-control environments and ten seeds per environment. The center is the mean normalized score and the band is a 95\% stratified-bootstrap interval obtained by resampling seeds within environments.}
  \label{fig:aggregate_learning}
\end{figure}

We additionally report an IQM plot (Figure~\ref{fig:aggregate_benchmark}) across 12 MuJoCo and Box2D environments using per-environment min--max normalization and stratified-bootstrap 95\% confidence intervals. AA-PDA and TD3 have overlapping confidence intervals, whereas SAC performs better overall. Because IQM is less sensitive to extreme values, these results suggest that AA-PDA's overall performance is comparable to TD3 under this statistic. The runtime comparison in Table~\ref{tab:runtime} shows that AA-PDA trains about 20--27 times faster than TD3. However, TD3 performs substantially better than AA-PDA in several individual environments, including Humanoid-v4, HalfCheetah-v4, and BipedalWalker-v3. In addition to the IQM results, Figure~\ref{fig:aggregate_learning} reports the aggregate mean normalized score across all 12 environments.


\clearpage

\end{document}